\documentclass{article}
\usepackage{ijcai22}

\usepackage{times}
\usepackage{soul}
\usepackage{url}
\usepackage[hidelinks]{hyperref}
\usepackage[utf8]{inputenc}
\usepackage[small]{caption}
\usepackage{graphicx}
\usepackage{amsmath}
\usepackage{amsthm}
\usepackage{booktabs}
\usepackage{algorithm}
\usepackage{algorithmic}
\usepackage{subfig}
\usepackage{amssymb}
\usepackage{multirow}
\usepackage{tabularx}
\usepackage{siunitx}
\usepackage{xfrac}
\usepackage{bigdelim}

\usepackage{tikz}
\usepackage{xcolor}
\newcommand{\cmt}[1]{\bgroup\hfill$\triangleright$~#1\egroup}

\newcommand*\circled[1]{\tikz[baseline=(char.base)]{
        \node[shape=circle,fill,inner sep=.9pt] (char) {\texttt{\textcolor{white}{#1}}};}}

\newtheorem{theorem}{Theorem}
\newtheorem{definition}{Definition}

\title{TinyLight: Adaptive Traffic Signal Control on Devices with Extremely Limited Resources}

\author{
    Dong Xing$^1$\and
    Qian Zheng$^1$\and
    Qianhui Liu$^{1,2}$\And
    Gang Pan$^{1,}$\footnote{Gang Pan is the corresponding author. }
    \affiliations
    $^1$College of Computer Science and Technology, Zhejiang University, Hangzhou, China\\
    $^2$Department of Electrical and Computer Engineering, National University of Singapore, Singapore
    \emails
    \{dongxing, qianzheng\}@zju.edu.cn, 
    qhliu@nus.edu.sg,
    gpan@zju.edu.cn
}

\begin{document}
    
    \maketitle
    
    \begin{abstract}
        Recent advances in deep reinforcement learning (DRL) have largely promoted the performance of adaptive traffic signal control (ATSC). Nevertheless, regarding the implementation, most works are cumbersome in terms of storage and computation. This hinders their deployment on scenarios where resources are limited. In this work, we propose TinyLight, the first DRL-based ATSC model that is designed for devices with extremely limited resources. TinyLight first constructs a super-graph to associate a rich set of candidate features with a group of light-weighted network blocks. Then,  to diminish the model's resource consumption, we ablate edges in the super-graph automatically with a novel entropy-minimized objective function. This enables TinyLight to work on a standalone microcontroller with merely 2KB RAM and 32KB ROM. We evaluate TinyLight on multiple road networks with real-world traffic demands. Experiments show that even with extremely limited resources, TinyLight still achieves competitive performance. The source code and appendix of this work can be found at \url{https://bit.ly/38hH8t8}. 
    \end{abstract}
    
    \section{Introduction}
    
    The latest progress of adaptive traffic signal control (ATSC) presents two different situations from views of {research} and {practice}. From the view of {research}, many studies on ATSC have shown superior performance over traditional fixed-cycle solutions \cite{survey_rl_on_atsc}. Especially, with the development of deep reinforcement learning (DRL)  \cite{xing21}, works on DRL-based ATSC exhibit great potential when applying them to intersections with various settings. From the view of {practice}, however, the deployment rate of ATSC in reality is still far from satisfactory. In 2019, less than 5\% of signalized intersections in the USA adopt ATSC systems \cite{tang2019global}. In developing regions, the application of ATSC is more uncommon \cite{ecoLight_Chauhan}. As traffic signal control is a field with broad application background, how to turn research outcomes into practice is critical. 
    
    One key factor impeding the application of ATSC is its high expense. For example, the cost of ATSC in Florida ranges from \$30,000 to \$96,400 per intersection \cite{elefteriadou2019before}. Recent studies on tiny machine learning enable simple algorithms such as kNN to work on cheap microcontroller units (MCUs) \cite{branco2019machine}. This inspires us to adopt such devices to partially reduce the cost of ATSC.\footnote{We choose MCU as it is a representative device with extremely limited resources, but our work is \textit{not} restricted to this device. } An MCU encapsulates processors, memory and I/O peripherals into a single chip. Comparing with a general workstation, the resources on MCUs are much limited (\textit{e.g.}, most MCUs have only KB of memory and mediocre clock rate). This makes its price relatively low (less than \$5 in our case) and therefore affordable in many practical scenarios. Nevertheless, its limitation on resource also poses new constraints on the {storage} and {computation} of ATSC implementations. 
    
    By interacting with the environment to adapt to real-time traffic flows, recent DRL-based solutions achieve superior performance over traditional ones. However, when considering their deployment, most works do not take into account the issue of resource consumption, leaving their performance on MCUs (or other embedding devices) undetermined. From the {storage} aspect, many works rely on complex network structures to maintain the representation power, and their memory requirement often exceeds the capacity of a single MCU. From the {computation} aspect, it is common for existing works to involve massive parallel operations, which greatly affect their response time if GPU is not available. These factors are practical during the deployment. However, we observe that they were rarely discussed in prior works. 
    
    To address these issues, we propose TinyLight, the first DRL-based ATSC model that is designed for devices with extremely limited resources. Specifically, TinyLight first constructs a super-graph to associate a rich set of candidate features with a group of light-weighted network blocks. This super-graph covers a wide range of sub-graphs that are compatible with resource-limited devices. Then, we ablate edges in the super-graph automatically with a novel entropy-minimized objective function, which greatly diminishes the resource consumption of TinyLight. In our experiments, the model we obtained takes only 4KB ROM. However, it still reaches competitive performance within 0.1s response time on devices with extremely low clock rate (8MHz). 
    
    As a demonstrative prototype of our model's deployment, we implement TinyLight on a standalone ATmega328P -- an MCU with merely 2KB RAM and 32KB ROM. This MCU is readily available on the market and costs less than \$5, making our work applicable in scenarios with low budgets. We evaluate TinyLight on multiple road networks with real-world traffic flows. Experiments show that even with extremely limited resources, TinyLight still achieves competitive performance. In this way, our work contributes to the transformation of ATSC research into practice. Although we focus on the task of traffic signal control, our methodology can be generalized to many fields that are also cost sensitive \cite{LiuRXT020}. 
    
    \section{Related Works}
    
    This section discusses related works of ATSC from both perspectives of research (works with no hardware involved) and practice (works with hardware involved). 
    
    \subsection{ATSC Research}
    
    Many pioneering works on ATSC are manually designed by experts in transportation engineering. For example, SOTL \cite{sotl} counts vehicles with or without the right of way and compares them with empirical thresholds. MaxPressure \cite{max_pressure} assumes downstream lanes have infinite capacity and greedily picks phase that maximizes the intersection's throughput. These works highly depend on human experience in transportation systems, making them too rigid for traffic flows that change dynamically. 
    
    Instead, DRL-based solutions directly learn from traffic flows and adapt to specific intersections more flexibly. To determine the next phase for a single intersection, \cite{frap_zheng} compares phases in pairs. \cite{attendlight} applies attention on both lanes and phases, and \cite{jiang_dynamic_lane} extends this idea to dynamic lanes. 
    For scenarios with multiple intersections, \cite{colight_wei} and \cite{macar} incorporate the influences of neighboring intersections with graph networks. \cite{rizzo} develops time critic policy gradient methods for congested networks. \cite{mplight_aaai} uses pressure to coordinate signals in region-level. Armed with judiciously designed networks, these models achieve superior performance over traditional ones. However, most of them are cumbersome in storage or computation, which impedes their deployment. In contrast, TinyLight is implemented with moderate scales of storage and computation. This enables our model to work on devices with limited resources, making it more practical for the deployment.
    
    \subsection{ATSC Practice}
    
    A number of commercial ATSC systems \cite{scoot} have been adopted in modern cities. However, the high costs of these systems constrain their popularity in most regions with tight budgets. Several works manage to narrow the gap between simulation and reality. They achieve this by designing communication platforms to evaluate the performance of general ATSC models in real-world environments \cite{Abdelgawad15}. However, these works do not change the implementation of ATSC models, leaving the cost of deployment still expensive. EcoLight \cite{ecoLight_Chauhan} is a threshold-based model for scenarios with primitive conditions. The authors first design a light-weighted network and then approximate it with a set of empirical thresholds to simplify the deployment. However, this process relies heavily on human expertise in the local traffic conditions, making it hard to generalize to other intersections. In comparison, the formation of TinyLight is fully data-driven, which greatly mitigates the model's dependency on manual operations. 
    
    Our work has also been influenced by recent works on neural architecture search \cite{elsken2019neural}. Nevertheless, existing studies in this field usually do not position the ultimate model size in the first place. This constraint, however, is the main focus of TinyLight. Unlike previous works, our design for the super-graph only involves light-weighted network blocks, and our proposed entropy-minimized objective function proactively ablates edges in the super-graph to a great extent. This makes our final model extremely tiny and therefore can  even work on devices with only KB of memory.

    \section{Preliminaries} 
    
    \begin{figure}
        \centering
        \includegraphics[width=0.8\linewidth]{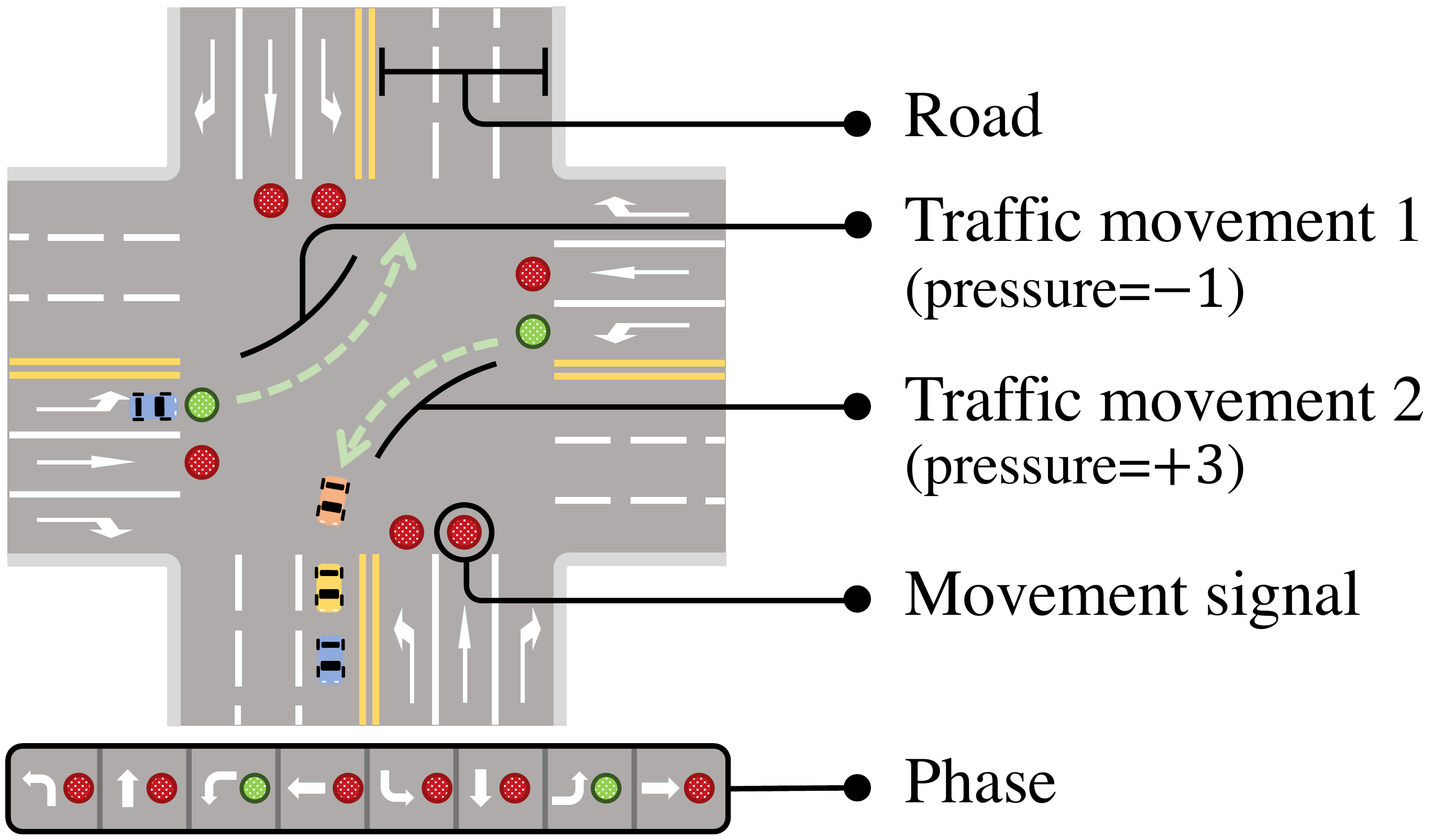}
        \caption{An illustration of the ATSC environment. }
        \label{fig:intersection}
    \end{figure}
    
    This section provides the definition of related terms and a formal description of the ATSC problem setting. 
    
    \begin{figure*}[!ht]
        \includegraphics[width=0.99\linewidth]{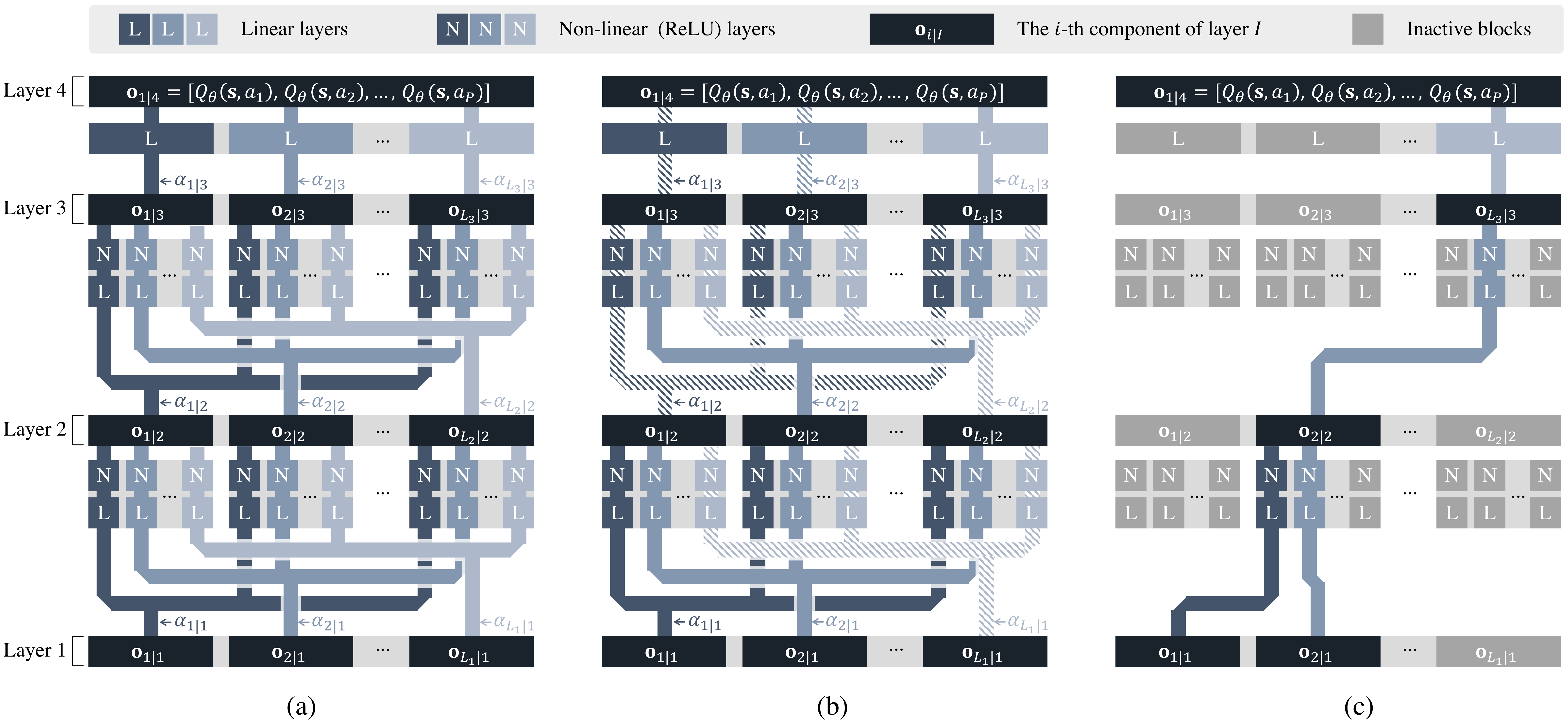}
        \caption{The formation process of TinyLight's final model. (a) The super-graph; (b) The intermediate result during edge ablating (stripped line indicates that the value of corresponding $\alpha$ is relatively smaller); (c) The sub-graph (\textit{i.e.} the policy network).  }
        \label{fig:model_nas}
    \end{figure*}
    
    \subsection{Term Definitions}
    
    We use a four-way intersection (depicted in Figure \ref{fig:intersection}) as an example to introduce the following terms.
    A \textit{road} is a wide way where vehicles can drive. A \textit{traffic movement} refers to a vehicle flow that starts from an incoming road and ends at an outgoing road. A \textit{movement signal} determines if a traffic movement has the right of way to pass through. A \textit{phase} is a combination of movement signals that do not conflict with each other. The \textit{pressure} measures the net difference between vehicles on the incoming and outgoing roads of a traffic movement. 
    These definitions are consistent with previous works \cite{survey_rl_on_atsc} and can be naturally extended to other intersection structures.

    \subsection{Problem Setting}
    
    Our problem follows the  setting of existing DRL-based works \cite{survey_rl_on_atsc}. At time step $t$, the model collects its state $\mathbf{s}_t$ and reward $r_t$ from the environment, based on which it determines a phase $a_t$ from the collection of $|P|$ valid phases to allow some traffic movements to pass through. The chosen phase is executed after 3 seconds of yellow signal to let drivers have enough time to prepare. The target of our problem is to maximize the cumulative discounted rewards $\mathbb{E} [\sum_t \gamma ^t r_t]$, where $\gamma$ is the discount factor of future return. For the implementation, we choose negative pressure over intersection at time $t$ as the instant reward $r_t$, which 
    suggests that intersections with lower pressures are more favored. This reward is correlated with long term performance indexes such as average travel time and throughput \cite{presslight_wei}.

    \section{Method}
    
    This section presents the details of our model. We first describe the implementation of TinyLight in Section \ref{sec:tinylight}, and then discuss its resource consumption in Section \ref{sec:resource}. 
    
    \subsection{The TinyLight Model} \label{sec:tinylight}
    
    \subsubsection{The Super-Graph} 
    Recall that our goal is to automatically extract a tiny but well-performed sub-graph from a super-graph, making it applicable on devices with extremely limited resources. Therefore, we expect the super-graph should cover a large quantity of tiny sub-graphs that have miscellaneous representation emphases. To meet this requirement, we implement a multi-scale feature extractor to represent a rich set of candidate features. This procedure depicts the road network from various scales, including (incoming and outgoing) lanes, roads, traffic movements, phases and the intersection. On each scale, a variety of traffic statistics are collected, such as the number of vehicles, the sum of their waiting time and so on. In consequence, 
    nearly 40 types of candidate features are prepared for the upcoming procedure of sub-graph extraction, which constitutes a comprehensive representation for the road network.\footnote{Appendix A provides the full list of our covered features. }  
    
    To further exploit these features, we construct a four-layer network structure, depicted in Figure \ref{fig:model_nas}a. The 1st layer is reserved for our candidate features. The 2nd and 3rd layers extract high-level representations. The last layer produces the phase. Each layer $I$ has $L_{I}$ parallel components (denoted as black rectangles) whose output dimensions differ from each other. During sub-graph extraction, only components that are most suitable for the current intersection are remained. For layer $J \in \{2, 3, 4\}$, its $j$-th component accumulates the transformation of all components from its previous layer $I$. Every single transformation involves a linear function and a non-linear (ReLU) function, which are light-weighted and can be implemented on resource-limited devices (c.f. Section \ref{subsec:analysis_resource}). It should be mentioned that the structure of our super-graph is not fixed and can be easily enhanced with deeper layers or stronger network blocks if the resource is abundant.
    
    \begin{figure*}[!t]
        \centering
        \subfloat[Hangzhou-1]{
            \setlength{\fboxsep}{0.5pt}\fbox{\includegraphics[width=0.153\linewidth]{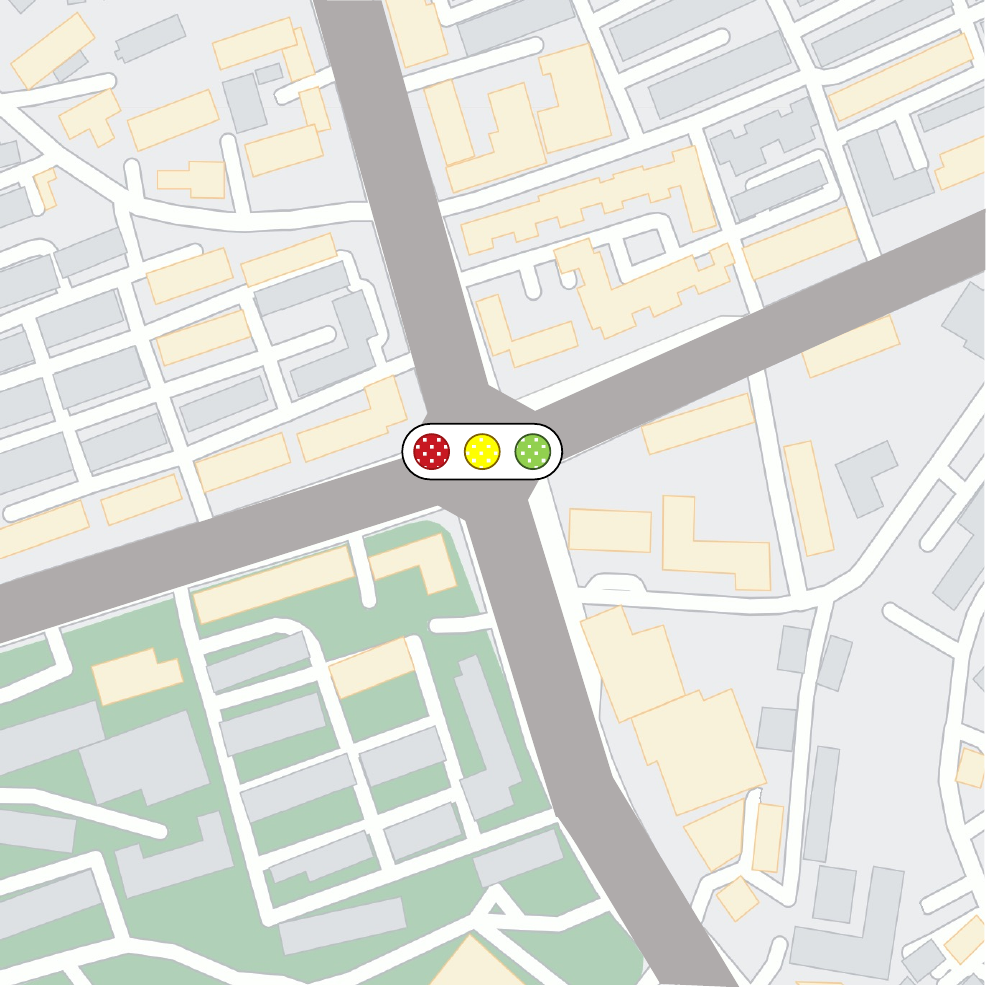}}
        }
        \subfloat[Hangzhou-2]{
            \setlength{\fboxsep}{0.5pt}\fbox{\includegraphics[width=0.153\linewidth]{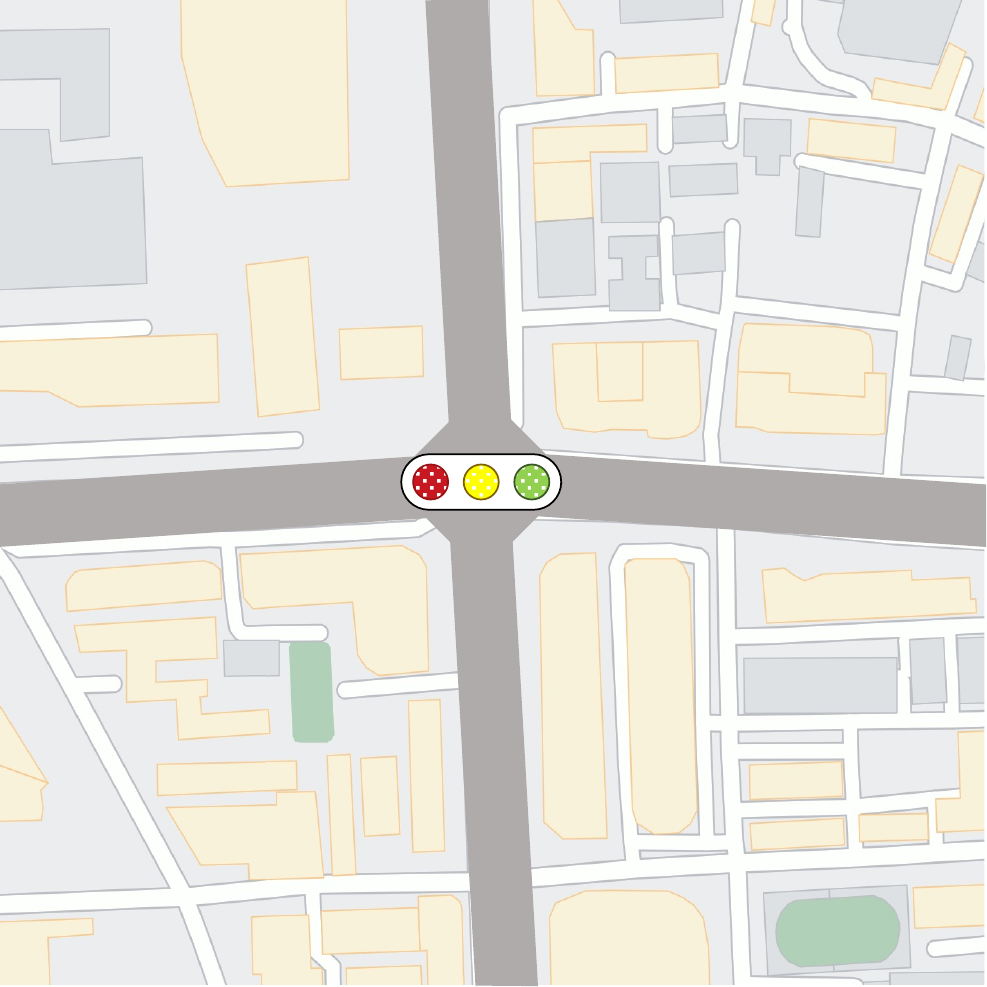}}
        }
        \subfloat[Hangzhou-3]{
            \setlength{\fboxsep}{0.5pt}\fbox{\includegraphics[width=0.153\linewidth]{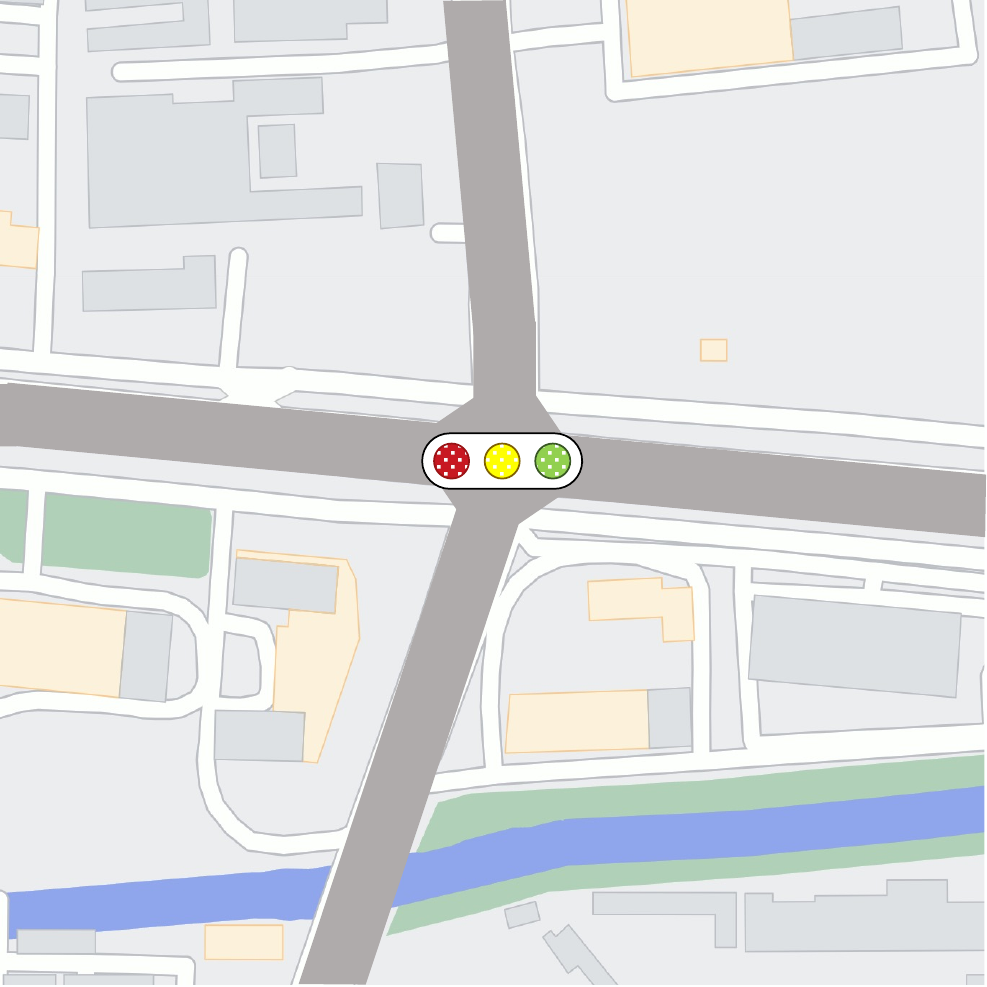}}
        }
        \subfloat[Atlanta]{
            \setlength{\fboxsep}{0.5pt}\fbox{\includegraphics[width=0.153\linewidth]{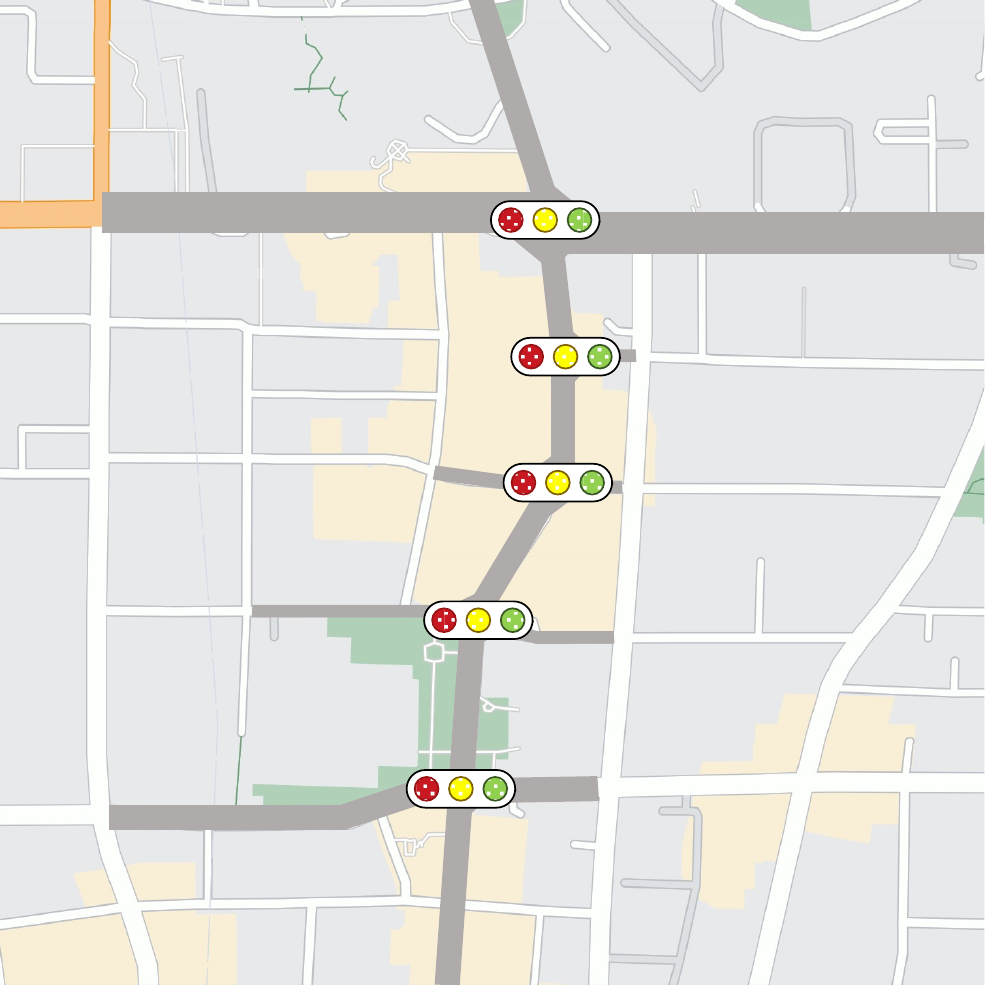}}
        }
        \subfloat[Jinan]{
            \setlength{\fboxsep}{0.5pt}\fbox{\includegraphics[width=0.153\linewidth]{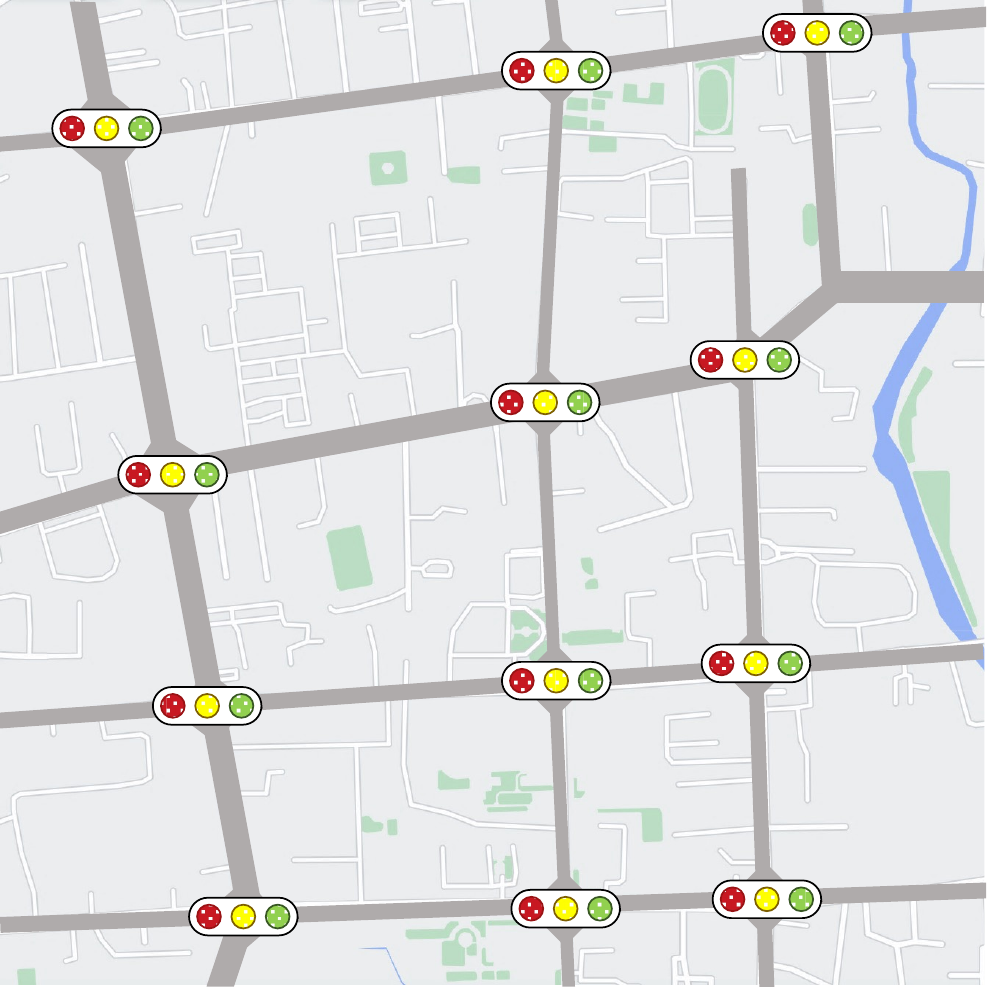}}
        }
        \subfloat[Los Angeles]{
            \setlength{\fboxsep}{0.5pt}\fbox{\includegraphics[width=0.153\linewidth]{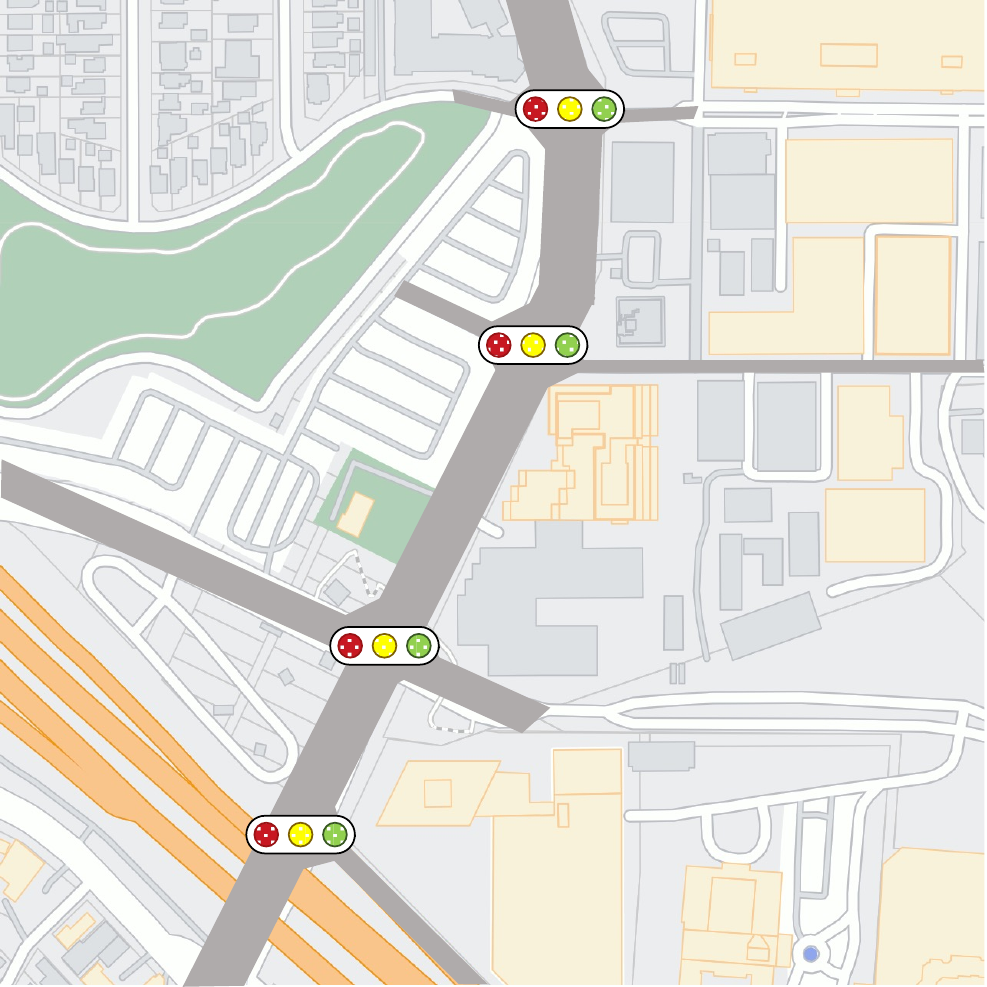}}
        }
        \caption{The road networks for evaluation. (a) Baochu Rd. and Tiyuchang Rd., Hangzhou; (b) Qingchun Rd. and Yan'an Rd., Hangzhou; (c) Tianmushan Rd. and Xueyuan Rd., Hangzhou; (d) Peachtree St., Atlanta; (e) Dongfeng Dist., Jinan; (f) Lankershim Blvd., Los Angeles. }
        \label{fig:dataset}
    \end{figure*}
    
    \subsubsection{The Sub-Graph}\label{sec:the sub-graph}
    To enable sub-graph extraction, we let all transformations that start from a common component (e.g., the $i$-th component of layer $I$ where $I \in \left\{1, 2, 3 \right\}$) be weighted by a shared non-negative parameter $\alpha_{i | I}$. On each layer $I$, the sum of all $\alpha_{i|I}$ equals to 1 (ensured by a softmax function). We denote the output of $j$-th component from layer $J$ as $\mathbf{o}_{j|J}$. Its relationship with $\mathbf{o}_{i|I}$ from the previous layer $I$ is represented as: 
    \begin{equation*}
        \mathbf{o}_{j| J} = \sum_{i} \alpha_{i | I} \cdot 
        \operatorname{ReLU} 
        \left(
        \operatorname{Linear}_{\theta} 
        \left(
        \mathbf{o}_{i | I}
        \right) 
        \right)
    \end{equation*}
    where $\theta$ denotes parameters of linear function. After extraction, most $\alpha$ will be set to 0, and their corresponding edges will be ablated. Since we restrict the sum of $\alpha_{i | I}$ to be 1 on each layer, there is at least one path that traverses from input to output layer, which guarantees the existence of our final sub-graph (proved in Appendix B.2). In addition, this design enables many sub-graphs to co-exist in the same super-graph. Even if the number of remaining connections on each layer is restricted to one, the proposed super-graph still covers a total amount of $\prod_I \left|L_I\right| $ sub-graphs. 
    
    \subsubsection{Two Parallel Objectives}\label{sec:two parallel objectives}
    
    Given the above settings, there are two groups of parameters to be optimized. The first is $\theta$ (parameters of all linear functions), which determines the performance of our sub-graph. The second is $\alpha$ (weights of all edges), which controls the model's sparsity. As our goal is to obtain a model that is both (i) well-performed and (ii) tiny in size, it can be naturally represented as a multi-objective optimization problem. 
    
    The first objective is to improve the performance of our sub-graph. This is made possible by minimizing: 
    \begin{equation*}
        \mathcal{L}_{\theta, \alpha}^{(1)} = 
        \mathbb{E} \left[
        \left(
        r_t + \gamma \max_{a_{t'}} Q_{\theta,\alpha} \left(
        \mathbf{s}_{t'}, a_{t'}
        \right)
        - Q_{\theta,\alpha} \left(
        \mathbf{s}_t, a_t 
        \right)
        \right) ^ 2
        \right]
    \end{equation*}
    which is the objective of value-based DRL under the structure defined by $\alpha$. Following convention, we use $Q_{\theta, \alpha}\left(\mathbf{s}_{t}, a_{t} \right)$ to denote the expected return when picking phase $a_t$ at state $\mathbf{s}_t$. It is an alias for the $a_t$-th dimension of $\mathbf{o}_{1 | I=4}$. By this objective, both $\theta$ and $\alpha$ are optimized to maximize the $Q$-value, which corresponds to models with better performance. 
    
    The second objective is to ablate edges in the super-graph proactively to diminish the size of our ultimate sub-graph. An ideal choice is to minimize the $\mathcal{L}_0$-norm of $\alpha$. However, this target is unfortunately intractable. To facilitate the optimization, we propose to minimize the sum of entropy of $\alpha_{i | I}$ on each layer $I$, which is represented as: 
    \begin{equation*}
        \mathcal{L}_{\alpha}^{(2)} 
        = \sum_{I} \mathcal{H} \left(\alpha_{ \cdot | I} \right)
        = 
        \sum_{I} \sum_{i} - \alpha_{i | I} \cdot \log \left(\alpha_{i | I}\right) 
    \end{equation*}
    Theoretically, $\mathcal{L}_{\alpha}^{(2)}$ is minimized when there is only one instance of $\alpha$ remained on each layer (proved in Appendix B.3). This makes the final size of our sub-graph controllable, even when the super-graph is designed to be enormously large. Comparing with $\mathcal{L}_0$-norm of $\alpha$, this objective is differentiable and therefore enables most gradient-based optimizers. 
    
    \subsubsection{The Optimization}
    By combining the aforementioned two parallel objectives with one Lagrangian function, we can reformulate our goal as the following entropy-minimized objective function:
    \begin{equation*}
        \mathcal{L}_{\theta, \alpha} = \mathcal{L}_{\theta, \alpha}^{(1)} + \beta \mathcal{L}_{\alpha}^{(2)}, ~\text{ s.t. }~ \beta \geq 0 
    \end{equation*}
    where $\beta$ is a non-negative Lagrangian multiplier. We perform the parameter updating with dual gradient descent, which alternates between the optimization of $\theta$ and $\alpha$. In this process, the $\alpha$ proactively ablates connections in the super-graph (Figure \ref{fig:model_nas}b), and the $\theta$ adjusts behavior of our sub-graph given its current structure. At the end of optimization, we remove most edges with negligible weights and eliminate all  network blocks that are not in the path to the output layer. We retain two features in layer 1 for a fair comparison with our comparing baselines. Nevertheless, to manifest that our model is tiny in size, we retain only one component in both layer 2 and 3. This determines the ultimate structure of our sub-graph, which is denoted in Figure \ref{fig:model_nas}c. Theoretically, its representation power is guaranteed by the universal approximation theorem \cite{universal_approximation_theorem}. Empirically, we observe that its performance remains competitive even when comparing with many strong baselines.\footnote{Appendix B provides more detailed settings for training. } 
    
    \begin{table*}[!t]
        \centering
        \begin{tabular}{l
                p{.07\linewidth}<{\raggedleft}@{}p{.05\linewidth}<{\raggedright}
                p{.06\linewidth}<{\raggedleft}@{}p{.05\linewidth}<{\raggedright}            
                c 
                p{.07\linewidth}<{\raggedleft}@{}p{.05\linewidth}<{\raggedright}
                p{.06\linewidth}<{\raggedleft}@{}p{.05\linewidth}<{\raggedright}            
                c 
                p{.07\linewidth}<{\raggedleft}@{}p{.05\linewidth}<{\raggedright}
                p{.06\linewidth}<{\raggedleft}@{}p{.05\linewidth}<{\raggedright}@{}            
            }
            \toprule 
            & \multicolumn{4}{c}{Hangzhou-1} && \multicolumn{4}{c}{Hangzhou-2} && \multicolumn{4}{c}{Hangzhou-3} \\
            \cmidrule{2-5} \cmidrule{7-10} \cmidrule{12-15}
            & \multicolumn{2}{c}{Travel Time} 
            & \multicolumn{2}{c}{Throughput}
            &
            & \multicolumn{2}{c}{Travel Time} 
            & \multicolumn{2}{c}{Throughput}
            &
            & \multicolumn{2}{c}{Travel Time} 
            & \multicolumn{2}{c}{Throughput} 
            \\ 
            & \multicolumn{2}{c}{\small{(sec./veh.)}}  
            & \multicolumn{2}{c}{\small{(veh./min.)}}
            &
            & \multicolumn{2}{c}{\small{(sec./veh.)}}  
            & \multicolumn{2}{c}{\small{(veh./min.)}}
            &
            & \multicolumn{2}{c}{\small{(sec./veh.)}}  
            & \multicolumn{2}{c}{\small{(veh./min.)}}   \\ 
            \midrule 
            EcoLight
            & 196.82& $_{\pm \text{18.69}}$
            & 29.73& $_{\pm \text{0.54}}$
            & 
            & 135.33& $_{\pm \text{2.33}}$
            & 23.16& $_{\pm \text{0.09}}$
            & 
            & 100.45& $_{\pm \text{2.42}}$
            & 28.67& $_{\pm \text{0.07}}$
            \\        
            FixedTime 
            &  282.68&  $_{\pm \text{2.01}}$
            &  27.59&  $_{\pm \text{0.04}}$
            &
            &  135.89 &  $_{\pm \text{1.94}}$
            & 23.23 &  $_{\pm \text{0.03}}$
            &
            & 418.05& $_{\pm \text{1.85}}$
            & 23.54& $_{\pm \text{0.03}}$
            \\ 
            MaxPressure
            & 121.23& $_{\pm \text{3.07}}$
            & 31.19& $_{\pm \text{0.09}}$
            & 
            & 138.72 &  $_{\pm \text{1.75}}$
            & 23.10 & $_{\pm \text{0.05}}$
            & 
            & 103.98& $_{\pm \text{2.41}}$
            & 28.62& $_{\pm \text{0.05}}$
            \\
            SOTL
            & 250.58& $_{\pm \text{3.66}}$
            & 28.69& $_{\pm \text{0.06}}$
            & 
            & 136.74& $_{\pm \text{2.02}}$
            & 23.21& $_{\pm \text{0.05}}$
            & 
            & 143.35& $_{\pm \text{3.55}}$
            & 28.32& $_{\pm \text{0.04}}$
            \\
            \midrule
            CoLight
            & \multicolumn{2}{c}{-} 
            & \multicolumn{2}{c}{-}
            & 
            & \multicolumn{2}{c}{-}
            & \multicolumn{2}{c}{-}
            & 
            & \multicolumn{2}{c}{-}
            & \multicolumn{2}{c}{-}
            \\
            FRAP
            & 129.65& $_{\pm \text{45.77}}$
            & 30.79& $_{\pm \text{1.21}}$
            & 
            & 137.02& $_{\pm \text{34.42}}$
            & 23.09& $_{\pm \text{0.37}}$
            & 
            & 107.74& $_{\pm \text{74.03}}$
            & 28.35& $_{\pm \text{1.20}}$
            \\
            MPLight
            & 128.61& $_{\pm \text{62.28}}$
            & 30.81& $_{\pm \text{1.16}}$
            & 
            & 121.87& $_{\pm \text{1.28}}$
            & 23.24& $_{\pm \text{0.06}}$
            & 
            & 82.48& $_{\pm \text{1.26}}$
            & 28.73& $_{\pm \text{0.05}}$
            \\
            TLRP
            & 152.88& $_{\pm \text{89.03}}$
            & 30.54& $_{\pm \text{1.38}}$
            & 
            & 126.70& $_{\pm \text{11.30}}$
            & 23.20& $_{\pm \text{0.15}}$
            & 
            & 83.98& $_{\pm \text{6.61}}$
            & 28.72& $_{\pm \text{0.07}}$
            \\
            \midrule
            TinyLight
            & \textbf{102.87} & $_{\pm \text{2.98}}$
            & \textbf{31.36}& $_{\pm \text{0.05}}$
            & 
            & \textbf{121.00}& $_{\pm \text{0.85}}$
            & \textbf{23.25}& $_{\pm \text{0.04}}$
            & 
            & \textbf{81.79}& $_{\pm \text{1.10}}$
            & \textbf{28.74}& $_{\pm \text{0.07}}$
            \\
            \bottomrule
            \toprule 
            & \multicolumn{4}{c}{Atlanta} &
            & \multicolumn{4}{c}{Jinan} &
            & \multicolumn{4}{c}{Los Angeles} 
            \\
            \cmidrule{2-5} \cmidrule{7-10} \cmidrule{12-15}
            & \multicolumn{2}{c}{Travel Time} 
            & \multicolumn{2}{c}{Throughput}
            &
            & \multicolumn{2}{c}{Travel Time} 
            & \multicolumn{2}{c}{Throughput}
            &
            & \multicolumn{2}{c}{Travel Time} 
            & \multicolumn{2}{c}{Throughput} 
            \\ 
            & \multicolumn{2}{c}{\small{(sec./veh.)}}  
            & \multicolumn{2}{c}{\small{(veh./min.)}}
            &
            & \multicolumn{2}{c}{\small{(sec./veh.)}}  
            & \multicolumn{2}{c}{\small{(veh./min.)}}
            &
            & \multicolumn{2}{c}{\small{(sec./veh.)}}  
            & \multicolumn{2}{c}{\small{(veh./min.)}}   \\ 
            \midrule 
            EcoLight
            & 303.07& $_{\pm \text{14.93}}$
            & 47.63& $_{\pm \text{5.26}}$
            & 
            & 384.43& $_{\pm \text{9.64}}$
            & 92.08& $_{\pm \text{0.79}}$
            & 
            & 659.01& $_{\pm \text{16.36}}$
            & 19.53& $_{\pm \text{1.12}}$
            \\
            FixedTime 
            & 297.13 & $_{\pm \text{3.19}}$  
            & 49.26 & $_{\pm \text{0.53}}$
            &
            &  457.21& $_{\pm \text{2.12}}$
            &  86.27& $_{\pm \text{0.17}}$
            &
            & 682.55 & $_{\pm \text{2.59}}$
            & 17.76 & $_{\pm \text{0.27}}$
            \\ 
            MaxPressure
            & 261.01& $_{\pm\text{4.20}}$
            & 59.44 & $_{\pm\text{1.51}}$
            & 
            & 340.13 & $_{\pm \text{1.69}}$
            & 94.76& $_{\pm \text{0.19}}$
            & 
            & 587.63 & $_{\pm \text{30.92}}$
            & 23.32& $_{\pm \text{2.69}}$
            \\
            SOTL
            & 416.88& $_{\pm \text{8.13}}$
            & 8.46 & $_{\pm \text{2.22}}$
            & 
            & 424.67& $_{\pm \text{2.44}}$
            & 90.02& $_{\pm \text{0.22}}$
            & 
            & 624.19& $_{\pm \text{29.48}}$
            & 21.81& $_{\pm \text{3.99}}$
            \\
            \midrule
            CoLight
            & \multicolumn{2}{c}{-} 
            & \multicolumn{2}{c}{-} 
            & 
            & 856.53& $_{\pm \text{451.32}}$
            & 57.96& $_{\pm \text{30.08}}$
            & 
            & \multicolumn{2}{c}{-}  
            & \multicolumn{2}{c}{-} 
            \\
            FRAP
            & 258.21& $_{\pm \text{18.27}}$
            & \textbf{63.49}& $_{\pm \text{8.19}}$
            & 
            & 327.90& $_{\pm \text{23.28}}$
            & 95.10& $_{\pm \text{1.17}}$
            & 
            & 737.36& $_{\pm \text{84.12}}$
            & 10.52& $_{\pm \text{5.98}}$
            \\
            MPLight
            & \multicolumn{2}{c}{-}
            & \multicolumn{2}{c}{-} 
            & 
            & \textbf{297.00} & $_{\pm \text{8.32}}$
            & \textbf{96.21} & $_{\pm \text{0.48}}$
            & 
            & \multicolumn{2}{c}{-} 
            & \multicolumn{2}{c}{-} 
            \\
            TLRP
            & 311.91& $_{\pm \text{35.18}}$
            & 41.10& $_{\pm \text{12.51}}$
            & 
            & 699.15& $_{\pm \text{224.30}}$
            & 65.74& $_{\pm \text{17.31}}$
            & 
            & 508.61& $_{\pm \text{88.51}}$
            & 30.30& $_{\pm \text{7.00}}$
            \\
            \midrule
            TinyLight
            & \textbf{253.99}& $_{\pm \text{8.08}}$
            & 62.16& $_{\pm \text{3.77}}$
            & 
            & 310.62& $_{\pm \text{3.68}}$
            & 95.79& $_{\pm \text{0.39}}$
            & 
            & \textbf{489.93}& $_{\pm \text{19.76}}$
            & \textbf{31.29}& $_{\pm \text{2.73}}$
            \\
            \bottomrule
            
        \end{tabular}
        \caption{The performance of all models on six real-world road networks (sec. = second, veh. = vehicle, min. = minute). }
        \label{table:result}
    \end{table*}
    
    \subsection{Resource Consumption}\label{sec:resource}
    
    In addition to the sparse connection of our sub-graph, there are several additional designs to further reduce its consumption on resources. First, although we adopt massive network blocks to implement the super-graph, all the basic operations are restricted to be in the set of \{linear, ReLU, summation\}. This enables us to implement all modules on embedding devices with no dependency on third-party libraries such as \texttt{gemmlowp} / \texttt{ruy} (required in TensorFlow Lite) or \texttt{QNNPACK} (required in PyTorch), which consume additional memory spaces. Second, we deliberately avoid using complex operators (such as convolution) that cause intensive computations in single-threaded programs. Instead, all our network blocks are chosen with moderate sizes of parameter and computation. This decreases our model's dependency on specialized devices such as GPU. We will quantitatively measure TinyLight's resource consumption in Section \ref{subsec:analysis_resource}. 
    
    \section{Experiments}\label{sec:experiment}
    
    This section presents experimental studies. Section \ref{subsec: experimental_setting} introduces related settings. Then, we evaluate TinyLight from both aspects of traffic efficiency in Section \ref{subsec:analysis_efficiency} and resource consumption in Section \ref{subsec:analysis_resource}, which correspond to our two parallel objectives. We also provide an ablation study for the sub-graph extraction of TinyLight in Section \ref{subsec:ablation_study}. The experiments are conducted on CityFlow \cite{cityflow}, an open-source simulator for realistic traffic environments. 
    
    \subsection{Experimental Settings}\label{subsec: experimental_setting}
    
    \paragraph{Datasets.} 
    We evaluate our work on six road networks with real-world traffic flows, which are all publicly available.\footnote{\url{https://traffic-signal-control.github.io/}} Figure \ref{fig:dataset} presents a top view of these road networks, including three single intersections from Hangzhou, one $5 \times 1$ network from Atlanta, one $4 \times 3$ network from Jinan and one $4\times 1$ network from Los Angeles. For datasets of Atlanta and Los Angeles, the intersection structures are heterogeneous. This prevents some baselines that require all intersections to be homogeneous. The traffic flows of Hangzhou and Jinan are obtained by surveillance cameras, and those of Atlanta and Los Angeles are collected from open vehicle trajectory data.\footnote{\url{https://ops.fhwa.dot.gov/trafficanalysistools/ngsim.htm}} To measure the statistical variations of each model, we additionally create nine traffic flows for each road network, which are obtained by shifting the initial appearance time of all vehicles in the real-world records of traffic flow with a random noise $\Delta t \in [-60s, 60s]$. 
    
    \paragraph{Baselines.} We adopt both rule-based ATSC solutions (EcoLight \cite{ecoLight_Chauhan}, FixedTime \cite{fixed_cycle}, MaxPressure \cite{max_pressure} and SOTL \cite{sotl}) and DRL-based ones (CoLight \cite{colight_wei}, FRAP \cite{frap_zheng} and MPLight \cite{mplight_aaai}) as our baselines. For the ablation study, we compare TinyLight with its variant that randomly selects paths to form the sub-graph, named TLRP (TinyLight with Random Path). 
    
    \subsection{Traffic Efficiency}\label{subsec:analysis_efficiency} We measure the traffic efficiency of all models with two indexes: the travel time (per vehicle) and the throughput (per minute). The first index evaluates the average time required for a vehicle to complete its trip. The second index counts on average how many vehicles complete their trips in a minute. 
    Both of these indexes are widely adopted in previous works on traffic signal control \cite{survey_rl_on_atsc}. 
    
    The results of all models on these two indexes are presented in Table \ref{table:result}. It can be observed that TinyLight achieves advanced performances on both indexes. Specifically, our model requires the lowest travel time on five datasets and holds the highest throughput on four of them. Even for datasets where TinyLight has not reached the best score, its performance still remains competitive. For instance, although the average throughput of TinyLight on Atlanta dataset is not the best, it only falls behind FRAP by a margin of $1.33$ vehicle per minute. On Jinan dataset, it is also the closest one to MPLight, the baseline which obtains the highest scores. Please note that comparing with all other DRL-based works, TinyLight is implemented with significantly fewer resource consumption (we will analyze this issue quantitatively in Section \ref{subsec:analysis_resource}). Nevertheless, our model still reaches competitive performance on these road networks with real-world traffic demands. This supports the feasibility of TinyLight: to find a model that is both well-performed and tiny in size, letting it implementable on cheap devices with limited resources. 
    
    \subsection{Resource Consumption}\label{subsec:analysis_resource}
    
    We analyze the resource consumption of each model with two criteria: storage and computation. They together determine the minimum configuration of the required device. We measure the model's storage by its parameter size and the model's computation by its amount of floating-point operations (FLOPs). These values are chosen because (i) they can be compared quantitatively, and (ii) they are independent of irrelevant conditions such as the clock rate or the optimization flags on compilers. These make them also frequently adopted in other resource-critic domains \cite{attention_is_all_you_need}. 
    
    Figure \ref{fig:resource} depicts the parameter sizes and FLOPs of each model on an intersection from Jinan dataset.\footnote{Appendix C provides the detailed computation procedure.} The values are presented in a log-log plot and the normalized scores of travel time for each model are represented with bars on top of each model. Comparing with rule-based models, although TinyLight has slightly larger scales of storage and computation (but still in the capacity of our target MCU), its performance is on average better than them. This is because rule-based works highly rely on human experience and can be too rigid when dealing with dynamic traffic flows. Comparing with other DRL-based models, TinyLight requires significantly fewer resources. Its storage is diminished by our entropy-minimized objective function, which ablates redundant edges proactively and makes its parameter size 19.32 times smaller than CoLight. As FRAP and MPLight are built up with convolutional operators, their parameters only exist in kernels and therefore are also small in scale. However, such operation incurs intensive parallel operations, which severely impact the model's response time ($>$1.5s on our target MCU) if GPU is not available. Instead, the computation of TinyLight is bounded by our super-graph, which intentionally avoids these complex operators and thus is 90.43 times smaller than FRAP and MPLight. To make these values more intuitive, we plot in Figure \ref{fig:resource} the borderlines of 32KB ROM and 0.1s response time on ATmega328P (an MCU with merely 2KB RAM, 32KB ROM and a clock rate of 8MHz), and TinyLight is the only DRL-based work which satisfies both conditions. This manifests the applicability of TinyLight on devices with extremely limited resources. 
    
    \begin{figure}
        \centering
        \includegraphics[width=0.84\linewidth]{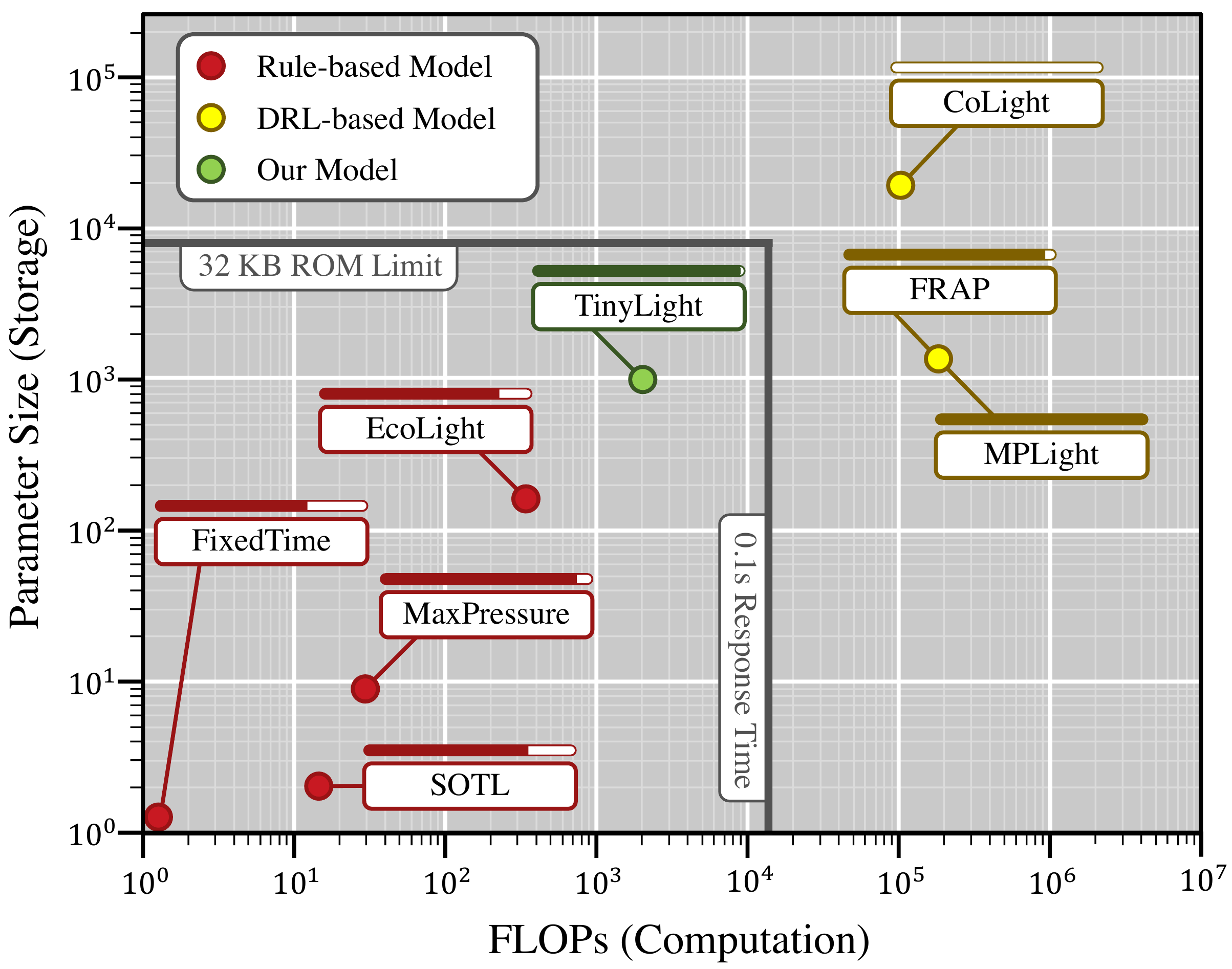}
        \caption{Log-log plot for each model's resource consumption. The bar on top of each model denotes its normalized score of travel time. }
        \label{fig:resource}
    \end{figure}

    \subsection{Ablation Study}\label{subsec:ablation_study}
    
    We attribute the performance of TinyLight to the procedure of sub-graph extraction, which helps us find a qualified policy network from massive candidates automatically. To support this argument, we further compare TinyLight with TLRP, a variant of itself that replaces the found path with randomly selected ones. Since many of our candidate features are obtained from previous works and our super-graph only involves basic operations, TLRP can be regarded as a model that is empirically designed with the constraint of limited resources. It can be concluded from Table \ref{table:result} that the variance of TLRP is much higher than that of TinyLight. For example, the variance of TLRP's travel time on Jinan dataset is 60.95 times higher than TinyLight. This reflects that manually designing a fixed network structure to fit all specific intersections is challenging. In contrast, TinyLight is able to adjust itself more smoothly to different intersections, as the model structure is automatically determined by the real traffic flows. 
    
    \section{Conclusion} 
    
    This paper presents TinyLight, the first DRL-based ATSC model that is designed for devices with extremely limited resources. 
    We first construct a super-graph to cover a wide range of sub-graphs that are compatible with resource-limited devices. Then, we ablate edges in the super-graph automatically with an entropy-minimized objective function. This enables TinyLight to work on a standalone MCU with merely KB of memory and mediocre clock rate, which costs less than \$5. We evaluate TinyLight on multiple road networks with real-world traffic demands and achieve competitive performance with significantly fewer resources. As traffic signal control is a field with broad application background, our work provides a feasible solution to apply ATSC on scenarios with limited budgets. Besides, our methodology can be trivially generalized to a wide range of fields that are cost sensitive. 
    
    We also acknowledge the limitations of our current solution and would like to point out some future directions for improvement. 
    First, although TinyLight can be implemented on embedding devices, its training still requires general-purposed machines. We expect the training won't be frequent as traffic pattern often repeats periodically. However, a better idea is to directly enable on-chip model re-training. Second, currently we are focusing on small-scaled problems as our interested scenarios are those with limited budgets. However, larger problems are also common in the world. Dealing with them is more challenging as the searching space is much wider. Therefore, a more effective sub-graph extraction method is expected under this condition \cite{YangZ021}.

    \section*{Acknowledgments}
    
    This work was supported by Natural Science Foundation of China (No. 61925603), The Key Research and Development Program of Zhejiang Province in China (2020C03004) and Zhejiang Lab.

    \bibliographystyle{named}
    \bibliography{ijcai22}
    
    \onecolumn
    \appendix
    \section{The Covered Features}\label{sec:candidate_feature}

    \subsection{The Formal Term Definitions}
    Before diving into the details, we first define some basic notations to let the definition of our features be more clear. To avoid confusion, we use bold letters to denote sets, and plain letters to denote elements in the set. Some notations are slightly abused to let the definitions more straightforward. 
    
    \begin{definition}[Intersection, road and lane]
        An intersection ${I}$ includes multiple incoming roads $\mathbf{R}_{\operatorname{in}} = \left\{ R_{\operatorname{in}}^{1}, R_{\operatorname{in}}^{2}, \cdots \right\} \subset {I}$ and outgoing roads $\mathbf{R}_{\operatorname{out}} = \left\{ R_{\operatorname{out}}^{1}, R_{\operatorname{out}}^{2}, \cdots \right\} \subset {I}$. The set of all roads that connects to the intersection ${I}$ is defined as $\mathbf{R} = \mathbf{R}_{\operatorname{in}} \cup \mathbf{R}_{\operatorname{out}} \subset {I}$. 
        For each incoming road $R_{\operatorname{in}} \in \mathbf{R}_{\operatorname{in}}$, there may exist multiple parallel incoming lanes $\mathbf{L}_{\operatorname{in}} = \left\{ L_{\operatorname{in}}^{1}, L_{\operatorname{in}}^{2}, \cdots \right\} \subseteq R_{\operatorname{in}}$. Similarly, there may exist multiple outgoing lanes $\mathbf{L}_{\operatorname{out}} = \left\{ L_{\operatorname{out}}^1, L_{\operatorname{out}}^2, \cdots \right\} \subseteq R_{\operatorname{out}}$ for each outgoing road $R_{\operatorname{out}} \in \mathbf{R}_{\operatorname{out}}$. 
    \end{definition}
    
    It is apparent that intersection, road and lane depict the road network from three different scales. Given the above definitions, we can now define (i) the set of lane links for the current intersection where vehicles can drive, and (ii) the phase which is used to operate the intersection. 
    
    \begin{definition}[Lane link] 
        A lane link $L_{\operatorname{link}}$ consists of a pair of incoming lane and outgoing lane, denoted as $L_{\operatorname{link}} = \left\{ L_{\operatorname{in}}, L_{\operatorname{out}} \right\}$. We only consider the set of valid lane links $\mathbf{L}_{\operatorname{link}}$, where vehicles are allowed to drive from the incoming lane $L_{\operatorname{in}}$ to the outgoing lane $L_{\operatorname{out}}$. 
    \end{definition}
    
    \begin{definition}[Phase]
        A phase $P$ allows vehicles from a set of lane links $\mathbf{L}_{\operatorname{link}}$ to pass through the intersection without introducing any right-of-way conflict. We denote the mapping from phase to passable lane links as 
        $
        P\left( {I} \right) = \mathbf{L}_{\operatorname{link}}
        $.  Given an intersection ${I}$, the set of all feasible phases $\mathbf{P} = \left\{   P^1, P^2, \cdots \right\}$ can be determined beforehand. 
    \end{definition}
    
    \begin{figure}[h]
        \centering
        \includegraphics[width=0.8\linewidth]{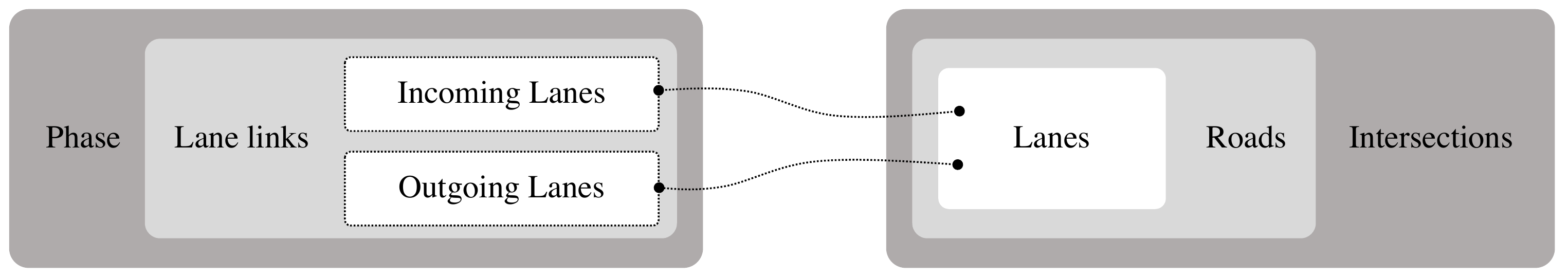}
        \caption{A diagram for the term relationship in our multi-scaled feature extractor. }
        \label{fig:venn}
    \end{figure}
    
    Figure \ref{fig:venn} depicts the relationship of aforementioned terms. We then define some basic functions to fetch the required information for our road network, which make the following illustrations more straightforward. 
    
    \begin{definition}[Vehicle]
        The function $f_{\operatorname{veh.}} \left( x \right)$  returns the set of vehicles for $x$, where $x$ can be any lane, road or intersection. 
    \end{definition}
    
    \begin{definition}[Speed]
        The function $f_{\operatorname{spd.}} \left(v \right)$ returns the current speed of vehicle $v$. 
    \end{definition}
    
    \begin{definition}[Delay]
        The function $f_{\operatorname{delay}}\left( L \right)$ returns the estimated delay for lane $L$ when comparing with the ideal condition without any traffic congestion. 
    \end{definition}
    
    \begin{definition}[Waiting time]
        The function $f_{\operatorname{wait}} (v)$ returns the waiting time for vehicle $v$ since its last stop. 
    \end{definition}
    
    \begin{definition}[Travel time]
        The function $f_{\operatorname{tt.}} (v)$ returns the travel time of vehicle $v$. 
    \end{definition}
    
    \begin{definition}[Bitmap image]
        The function $f_{\operatorname{img.}} (I)$ returns a bitmap representation for the intersection $I$. 
    \end{definition}
    
    \subsection{The List of Covered Features}
    Here we  provide the complete list of candidate features implemented in our work. Some of them also appear in previous DRL-based works (c.f. \cite{survey_rl_on_atsc} for an up-to-date survey), and the others are their natural extensions under different scales. 
    
    \begin{enumerate}
        \item \texttt{lane\_2\_num\_vehicle}: the mapping from each lane to the number of vehicles on that lane, defined as:
        \begin{equation*}
            F_1 = \left\{ L \rightarrowtail 
            \left| f_{\operatorname{veh.}}\left( L\right) \right|   \mid \forall L \in \mathbf{L} \right\} 
        \end{equation*} 
        \item \texttt{lane\_2\_num\_waiting\_vehicle}: the mapping from each lane to the number of vehicles that are currently waiting on that lane, defined as:
        \begin{equation*}
            F_2 = 
            \left\{ 
            L \rightarrowtail 
            \left| 
            \left\{ 
            v \in  f_{\operatorname{veh.}} \left(L \right) \mid f_{\operatorname{spd.}}\left( v \right) = 0 
            \right\} 
            \right|   \mid 
            \forall L \in \mathbf{L}
            \right\} 
        \end{equation*} 
        \item \texttt{lane\_2\_sum\_waiting\_time}: the mapping from each lane to the sum of waiting time of all vehicles on that lane since their last stop, defined as: 
        \begin{equation*}
            F_3 = \left\{ 
            L \rightarrowtail  
            \sum f_{\operatorname{wait}} \left(v \right) , \forall v \in f_{\operatorname{veh.}} \left( L \right) 
            \mid 
            \forall L \in \mathbf{L}
            \right\} 
        \end{equation*}
        \item \texttt{lane\_2\_delay}: the mapping from each lane to the estimated delay on that lane. Defined as: 
        \begin{equation*}
            F_4 = 
            \left\{
            L \rightarrowtail 
            f_{\operatorname{delay}} \left( L \right) 
            \mid \forall L \in \mathbf{L}
            \right\} 
        \end{equation*} 
        \item \texttt{lane\_2\_num\_vehicle\_seg\_by\_k}: this feature first splits each lane into $K$ segments and then computes the mapping from each segment to the number of vehicles on that segment. Defined as: 
        \begin{equation*}
            F_5 = \left\{ L^{k} \rightarrowtail 
            \left| f_{\operatorname{veh.}}\left( L^k \right) \right|   \mid \forall L \in \mathbf{L}, k \in \left[1, \cdots, K \right] \right\} 
        \end{equation*} 
        \item \texttt{inlane\_2\_num\_vehicle}: similar to \texttt{lane\_2\_num\_vehicle}, but only consider incoming lanes. Defined as: 
        \begin{equation*}
            F_6 = \left\{ 
            L_{\operatorname{in}} \rightarrowtail \left| 
            f_{\operatorname{veh.}} \left( L_{\operatorname{in}} \right)
            \right| 
            \mid 
            \forall L_{\operatorname{in}} \in \mathbf{L}_{\operatorname{in}}
            \right\} 
        \end{equation*}
        \item \texttt{inlane\_2\_num\_waiting\_vehicle}: similar to \texttt{lane\_2\_num\_waiting\_vehicle}, but only consider incoming lanes. Defined as: 
        \begin{equation*}
            F_7 = 
            \left\{ 
            L_{\operatorname{in}} \rightarrowtail 
            \left| 
            \left\{ 
            v \in  f_{\operatorname{veh.}} \left(L_{\operatorname{in}} \right) \mid f_{\operatorname{spd.}}\left( v \right) = 0 
            \right\} 
            \right|   \mid 
            \forall L_{\operatorname{in}} \in \mathbf{L}_{\operatorname{in}}
            \right\} 
        \end{equation*} 
        
        \item \texttt{inlane\_2\_sum\_waiting\_time}: similar to \texttt{lane\_2\_sum\_waiting\_time}, but only consider incoming lanes. Defined as: 
        \begin{equation*}
            F_8 = \left\{ 
            L_{\operatorname{in}} \rightarrowtail  
            \sum f_{\operatorname{wait}} \left(v \right) , \forall v \in f_{\operatorname{veh.}} \left( L_{\operatorname{in}} \right) 
            \mid 
            \forall L_{\operatorname{in}} \in \mathbf{L}_{\operatorname{in}}
            \right\} 
        \end{equation*}
        \item \texttt{inlane\_2\_delay}: similar to \texttt{lane\_2\_delay}, but only consider incoming lanes. Defined as: 
        \begin{equation*}
            F_9 = 
            \left\{
            L_{\operatorname{in}} \rightarrowtail 
            f_{\operatorname{delay}} \left( L_{\operatorname{in}} \right) 
            \mid \forall L_{\operatorname{in}} \in \mathbf{L}_{\operatorname{in}}
            \right\} 
        \end{equation*} 
        \item \texttt{inlane\_2\_num\_vehicle\_seg\_by\_k}: similar to \texttt{lane\_2\_num\_vehicle\_seg\_by\_k}, but only consider incoming lanes. Defined as: 
        \begin{equation*}
            F_{10} = \left\{ L_{\operatorname{in}}^{k} \rightarrowtail 
            \left| f_{\operatorname{veh.}}\left( L_{\operatorname{in}}^k \right) \right|   \mid \forall L_{\operatorname{in}} \in \mathbf{L}_{\operatorname{in}}, k \in \left[1, \cdots, K \right] \right\} 
        \end{equation*} 
        \item \texttt{inlane\_2\_pressure}: the mapping from each incoming lane to its pressure. Defined as: 
        \begin{equation*}
            F_{11} = \left\{ 
            L_{\operatorname{in}} \rightarrowtail 
            \sum \left( f_{\operatorname{veh.}} \left( L_{\operatorname{out}}\right) - 
            f_{\operatorname{veh.}} \left( L _{\operatorname{in}}\right)  \right) \text{ where } 
            \left\{ L_{\operatorname{in}}, L_{\operatorname{out}} \right\} \in \mathbf{L}_{\operatorname{link}} 
            \mid 
            \forall                L_{\operatorname{in}} \in \mathbf{L}_{\operatorname{in}}
            \right\} 
        \end{equation*}
        \item \texttt{outlane\_2\_num\_vehicle}: similar to \texttt{lane\_2\_num\_vehicle}, but only consider outgoing lanes.  Defined as: 
        \begin{equation*}
            F_{12} = \left\{ 
            L_{\operatorname{out}} \rightarrowtail \left| 
            f_{\operatorname{veh.}} \left( L_{\operatorname{in}} \right)
            \right| 
            \mid 
            \forall L_{\operatorname{out}} \in \mathbf{L}_{\operatorname{out}}
            \right\} 
        \end{equation*}
        \item \texttt{outlane\_2\_num\_waiting\_vehicle}: similar to \texttt{lane\_2\_num\_waiting\_vehicle}, but only consider outgoing lanes. Defined as: 
        \begin{equation*}
            F_{13} = 
            \left\{ 
            L_{\operatorname{out}} \rightarrowtail 
            \left| 
            \left\{ 
            v \in  f_{\operatorname{veh.}} \left(L_{\operatorname{out}} \right) \mid f_{\operatorname{spd.}}\left( v \right) = 0 
            \right\} 
            \right|   \mid 
            \forall L_{\operatorname{out}} \in \mathbf{L}_{\operatorname{out}}
            \right\} 
        \end{equation*} 
        \item \texttt{outlane\_2\_sum\_waiting\_vehicle}: similar to \texttt{lane\_2\_sum\_waiting\_vehicle}, but only consider outgoing lanes. Defined as: 
        \begin{equation*}
            F_{14} = \left\{ 
            L_{\operatorname{out}} \rightarrowtail  
            \sum f_{\operatorname{wait}} \left(v \right) , \forall v \in f_{\operatorname{veh.}} \left( L_{\operatorname{out}} \right) 
            \mid 
            \forall L_{\operatorname{out}} \in \mathbf{L}_{\operatorname{out}}
            \right\} 
        \end{equation*}
        \item \texttt{outlane\_2\_delay}: similar to \texttt{lane\_2\_delay}, but only consider outgoing lanes. Defined as: 
        \begin{equation*}
            F_{15} = 
            \left\{
            L_{\operatorname{out}} \rightarrowtail 
            f_{\operatorname{delay}} \left( L_{\operatorname{out}} \right) 
            \mid \forall L_{\operatorname{out}} \in \mathbf{L}_{\operatorname{out}}
            \right\} 
        \end{equation*} 
        \item \texttt{outlane\_2\_num\_vehicle\_seg\_by\_k}:  similar to \texttt{lane\_2\_num\_vehicle\_seg\_by\_k}, but only consider outgoing lanes. Defined as: 
        \begin{equation*}
            F_{16} = \left\{ L_{\operatorname{out}}^{k} \rightarrowtail 
            \left| f_{\operatorname{veh.}}\left( L_{\operatorname{out}}^k \right) \right|   \mid \forall L_{\operatorname{out}} \in \mathbf{L}_{\operatorname{out}}, k \in \left[1, \cdots, K \right] \right\} 
        \end{equation*} 
        \item \texttt{inroad\_2\_num\_vehicle}: the mapping from each incoming road to the number of vehicles on that road. Defined as: 
        \begin{equation*}
            F_{17} = \left\{
            R_{\operatorname{in}} \rightarrowtail 
            \left| 
            f_{\operatorname{veh.}} \left( R_{\operatorname{in}} \right)
            \right| 
            \mid 
            \forall R_{\operatorname{in}} \in \mathbf{R}_{\operatorname{in}}
            \right\} 
        \end{equation*}
        \item \texttt{inroad\_2\_num\_waiting\_vehicle}: the mapping from each incoming road to the number of vehicles that are waiting on that road. Defined as: 
        \begin{equation*}
            F_{18} = 
            \left\{ 
            R_{\operatorname{in}} \rightarrowtail 
            \left| 
            \left\{ 
            v \in  f_{\operatorname{veh.}} \left(R_{\operatorname{in}} \right) \mid f_{\operatorname{spd.}}\left( v \right) = 0 
            \right\} 
            \right|   \mid 
            \forall R_{\operatorname{in}} \in \mathbf{R}_{\operatorname{in}}
            \right\} 
        \end{equation*} 
        \item \texttt{inroad\_2\_sum\_waiting\_time}: the mapping from each incoming road to the sum of waiting time of all vehicles on that road since their last stop. Defined as: 
        \begin{equation*}
            F_{19} = \left\{ 
            R_{\operatorname{in}} \rightarrowtail  
            \sum f_{\operatorname{wait}} \left(v \right) , \forall v \in f_{\operatorname{veh.}} \left( R_{\operatorname{in}} \right) 
            \mid 
            \forall R_{\operatorname{in}} \in \mathbf{R}_{\operatorname{in}}
            \right\} 
        \end{equation*}
        \item \texttt{inroad\_2\_delay}: the mapping from each incoming road to the average delay of all lanes on that road. Defined as: 
        \begin{equation*}
            F_{20} = 
            \left\{
            R_{\operatorname{in}} \rightarrowtail 
            f_{\operatorname{delay}} \left( R_{\operatorname{in}} \right) 
            \mid \forall R_{\operatorname{in}} \in \mathbf{R}_{\operatorname{in}}
            \right\} 
        \end{equation*} 
        \item \texttt{phase\_2\_num\_vehicle}: the mapping from each phase to the number of vehicles that have the right of way under  that phase.  Defined as: 
        \begin{equation*}
            F_{21} = \left\{ 
            P \rightarrowtail 
            \sum  f_{\operatorname{veh.}} (L_{\operatorname{link}}),  
            \forall   L_{\operatorname{link}} \in P\left({I}\right)
            \mid 
            \forall   P \in \mathbf{P}  
            \right\} 
        \end{equation*}
        \item \texttt{phase\_2\_num\_waiting\_vehicle}: the mapping from each phase to the number of vehicles that will terminate their waiting state under that phase. Defined as:
        \begin{equation*}
            F_{22} = 
            \left\{ 
            P \rightarrowtail 
            \left| 
            \left\{ 
            v \in  P \left({I} \right) \mid f_{\operatorname{spd.}}\left( v \right) = 0 
            \right\} 
            \right|   \mid 
            \forall P \in \mathbf{P}
            \right\} 
        \end{equation*} 
        \item \texttt{phase\_2\_sum\_waiting\_time}: the mapping from each phase to the sum of waiting time for all vehicles that are about to obtain the right of way under that phase. Defined as: 
        \begin{equation*}
            F_{23} = \left\{ 
            P \rightarrowtail  
            \sum f_{\operatorname{wait}} \left(v \right) , \forall v \in P\left( {I} \right) 
            \mid 
            \forall P \in \mathbf{P}
            \right\} 
        \end{equation*}
        \item \texttt{phase\_2\_delay}: the mapping from each phase to the average delay of all lanes that have the right of way under that phase. Defined as: 
        \begin{equation*}
            F_{24} = 
            \left\{
            P \rightarrowtail 
            f_{\operatorname{delay}} \left( P\left( {I }\right) \right) 
            \mid \forall P \in \mathbf{P}
            \right\} 
        \end{equation*} 
        \item \texttt{phase\_2\_pressure}: the mapping from each phase to the average pressure over all lane links that have the right of way under that phase. Defined as: 
        \begin{equation*}
            F_{25} = \left\{ 
            P \rightarrowtail 
            \sum \left( f_{\operatorname{veh.}} \left( L_{\operatorname{out}}\right) - 
            f_{\operatorname{veh.}} \left( L _{\operatorname{in}}\right)  \right) \text{ where } 
            \left\{ L_{\operatorname{in}}, L_{\operatorname{out}} \right\} \in P({I} )
            \mid 
            \forall            P \in \mathbf{P}
            \right\} 
        \end{equation*}
        \item \texttt{inter\_2\_num\_vehicle}: the total number of vehicles on that intersection. Defined as: 
        \begin{equation*}
            F_{26} = \left\{ I \rightarrowtail 
            \left| f_{\operatorname{veh.}}\left( I\right) \right|   \mid \forall I \in \mathbf{I} \right\} 
        \end{equation*} 
        \item \texttt{inter\_2\_num\_waiting\_vehicle}: the total number of vehicles that are waiting on that intersection. Defined as: 
        \begin{equation*}
            F_{27}= 
            \left\{ 
            I  \rightarrowtail 
            \left| 
            \left\{ 
            v \in  f_{\operatorname{veh.}} \left(I \right) \mid f_{\operatorname{spd.}}\left( v \right) = 0 
            \right\} 
            \right|   \mid 
            \forall I \in \mathbf{I}
            \right\} 
        \end{equation*} 
        \item \texttt{inter\_2\_sum\_waiting\_time}: the sum of waiting time of all vehicles on that intersection since their last halt. Defined as: 
        \begin{equation*}
            F_{28} = \left\{ 
            I \rightarrowtail  
            \sum f_{\operatorname{wait}} \left(v \right) , \forall v \in f_{\operatorname{veh.}} \left( I \right) 
            \mid 
            \forall I \in \mathbf{I}
            \right\} 
        \end{equation*}
        \item \texttt{inter\_2\_delay}: the average delay of all lanes on that intersection. Defined as: 
        \begin{equation*}
            F_{29} = 
            \left\{
            I \rightarrowtail 
            f_{\operatorname{delay}} \left( I \right) 
            \mid \forall I \in \mathbf{I}
            \right\} 
        \end{equation*} 
        \item \texttt{inter\_2\_pressure}: the average pressure of all lane links on that intersection. 
        Defined as: 
        \begin{equation*}
            F_{30} = \left\{ 
            I \rightarrowtail 
            \sum \left( f_{\operatorname{veh.}} \left( L_{\operatorname{out}}\right) - 
            f_{\operatorname{veh.}} \left( L _{\operatorname{in}}\right)  \right) \text{ where } 
            \left\{ L_{\operatorname{in}}, L_{\operatorname{out}} \right\} \in I 
            \mid 
            \forall            I \in \mathbf{I}
            \right\} 
        \end{equation*}
        \item \texttt{inter\_2\_vehicle\_position\_image}: an image for the position of all vehicles on that intersection. Defined as: 
        \begin{equation*}
            F_{31} = \left\{
            I \rightarrowtail f_{\operatorname{img}} \left( I \right)
            \mid 
            \forall I \in \mathbf{I}
            \right\} 
        \end{equation*}
        \item \texttt{inter\_2\_current\_phase}: the one-hot representation of current phase on that intersection.\footnote{$\mathbf{1}\left(\cdot\right)$ denotes the indicator function, which returns 1 if the condition is true, otherwise it returns 0. } Defined as: 
        \begin{equation*}
            F_{32} = \left\{ 
            I \rightarrowtail 
            \left[
            \mathbf{1}\left(P_{i=1}\right), 
            \mathbf{1}\left(P_{i=2}\right), 
            \cdots , 
            \mathbf{1}\left(P_{i=\left| P\right| }\right)
            \right]
            \mid 
            \forall                I \in \mathbf{I}
            \right\} 
        \end{equation*}
        \item \texttt{inter\_2\_phase\_has\_changed}: an indicator  to show that whether the phase has been changed at last time.\footnote{We use the superscript $t$ to denote the instance at time step $t$.} Defined as: 
        \begin{equation*}
            F_{33} = \left\{ 
            I \rightarrowtail 
            \mathbf{1} \left( P^{t-1}({I}) = P^t(I) \right)
            \mid 
            \forall I \in \mathbf{I} 
            \right\} 
        \end{equation*}
        \item \texttt{inter\_2\_num\_passed\_vehicle\_since\_last\_action}: the total number of vehicles that pass the intersection from last action to now. Defined as: 
        \begin{equation*}
            F_{34} = \left\{ 
            I \rightarrowtail  
            \left| f_{\operatorname{veh.}}^{t} \left( I\right) \right| 
            - 
            \left| f_{\operatorname{veh.}}^{t-1} \left( I\right) \right| 
            \mid 
            \forall                I \in \mathbf{I}
            \right\} 
        \end{equation*}
        \item \texttt{inter\_2\_sum\_travel\_time\_since\_last\_action}: the total sum of travel time for vehicles that pass the intersection from last time to now. Defined as: 
        \begin{equation*}
            F_{35}= \left\{ 
            I \rightarrowtail 
            \sum f_{\operatorname{tt}} \left( v \right) \text{ where } v \in I^{t - 1} \text{ and } v \notin I^{t} 
            \mid 
            \forall            I \in \mathbf{I}
            \right\} 
        \end{equation*}
        \item \texttt{lanelink\_2\_pressure}: the mapping from each lane link to its pressure. Defined as: 
        \begin{equation*}
            F_{36} = \left\{ 
            L_{\operatorname{link}} = \left\{ L_{\operatorname{in}}, L_{\operatorname{out}} \right\} \rightarrowtail 
            \left|   f_{\operatorname{veh.}} \left( L_{\operatorname{out}}\right)  \right| - 
            \left|  f_{\operatorname{veh.}} \left( L _{\operatorname{in}}\right)  \right| 
            \mid 
            \forall            L_{\operatorname{link}} \in I 
            \right\} 
        \end{equation*}
        \item \texttt{lanelink\_2\_num\_vehicle}: the mapping from each lane link to the number of vehicles on that lane link. Defined as: 
        \begin{equation*}
            F_{37} = \left\{ 
            L_{\operatorname{link}} = \left\{ L_{\operatorname{in}}, L_{\operatorname{out}} \right\} \rightarrowtail 
            \left| 
            f_{\operatorname{veh.}} \left( L_{\operatorname{link}} \right)
            \right| 
            \mid 
            \forall            L_{\operatorname{link}} \in I 
            \right\} 
        \end{equation*}
    \end{enumerate}
    
    \newpage 
    \section{Model Specification}\label{sec:model_specification}
    
    \subsection{The Pseudo Code} 
    
    The pseudo code for our implementation is presented in Algorithm \ref{alg:algorithm}. 
    
    \begin{algorithm}
        \caption{TinyLight}
        \label{alg:algorithm}
        \textbf{Input}: $\mathcal{G}_{\operatorname{super}}$  (the super-graph)\\
        \textbf{Output}: $\mathcal{G}_{\operatorname{sub}}$  (the sub-graph)
        \begin{algorithmic}[1] 
            \STATE /* Search the sub-graph */
            \WHILE{not converged at episode $j$}
            \STATE Reset the environment \texttt{env}
            \WHILE{ episode $j$ has not terminated at time $t$} 
            \STATE $a_t = \mathcal{G}_{\operatorname{super}} \left( \mathbf{s}_t \right)$  \cmt{Sample a phase $a_t$ from $\mathcal{G}_{\operatorname{super}}$} 
            \STATE $\mathbf{s}_{t'}, r_t = \mathtt{env.step} \left( a_t \right)$  \cmt{Interact with the \texttt{env}}
            \STATE Update the replay buffer 
            \IF{samples are sufficient} 
            \STATE Sample a mini-batch $\mathcal{S}$ from the replay buffer
            \STATE Update $\theta$ by:   \cmt{Objective 1: improve the performance of $\mathcal{G}_{\operatorname{sub}}$}
            \begin{equation}\label{eq:theta}
                \theta \leftarrow  \arg\min_{\theta}   \mathcal{L}_{\theta, \alpha}^{(1)} = 
                \mathbb{E}_{\mathcal{S}} \left[
                \left(
                r_t + \gamma \max_{a_{t+1}} Q_{\theta,\alpha} \left(
                \mathbf{s}_{t+1}, a_{t+1}
                \right)
                - Q_{\theta,\alpha} \left(
                \mathbf{s}_t, a_t
                \right)
                \right) ^ 2
                \right]
            \end{equation}
            \STATE Update $\alpha$ by: \cmt{Objective 2: diminish the size of $\mathcal{G}_{\operatorname{sub}}$}
            \begin{equation}
                \alpha 
                \leftarrow 
                \arg\min_{\alpha} \mathcal{L}_{\theta,\alpha}^{(1)}  +  \beta \mathcal{L}_{\alpha}^{(2)} 
                =  
                \mathbb{E}_{\mathcal{S}} \left[
                \left(
                r_t + \gamma \max_{a_{t'}} Q_{\theta,\alpha} \left(
                \mathbf{s}_{t'}, a_{t'}
                \right)
                - Q_{\theta,\alpha} \left(
                \mathbf{s}_t, a_t 
                \right)
                \right) ^ 2
                \right]
                +  \beta \sum_{I}  \sum_{i} -
                \alpha_{i | I }
                \cdot 
                \log 
                \left(
                \alpha_{i | I }
                \right)
            \end{equation}
            \STATE Update the target network 
            \ENDIF 
            \ENDWHILE 
            \STATE Update the exploration rate $\epsilon$
            \ENDWHILE 
            
            \STATE 
            \STATE /* Refine the sug-graph */
            \STATE Extract the sub-graph $\mathcal{G}_{\operatorname{sub}}$ from $\mathcal{G}_{\operatorname{super}}$ according to the value of $\alpha$ on each layer
            \WHILE{not converged at episode $j$} 
            \STATE Reset the environment \texttt{env}
            \WHILE{episode $j$ has not terminated at time $t$} 
            \STATE $a_t = \mathcal{G}_{\operatorname{sub}} \left( \mathbf{s}_t \right)$  \cmt{Sample a phase $a_t$ from $\mathcal{G}_{\operatorname{sub}}$} 
            \STATE $\mathbf{s}_{t'}, r_t = \mathtt{env.step} \left( a_t \right)$  \cmt{Interact with the \texttt{env}}
            \STATE Update replay buffer 
            \STATE Sample a mini-batch $\mathcal{S}$ and update $\theta$ by Eq. \ref{eq:theta}  \cmt{Refine the sub-graph $\mathcal{G}_{\operatorname{sub}}$}
            \ENDWHILE 
            \ENDWHILE 
            
            \STATE \textbf{return} solution
        \end{algorithmic}
    \end{algorithm}
    
    \subsection{Proof for the Existence of Sub-Graph}
    
    \begin{theorem}
        Given the structure of our super-graph and the constraint of $\alpha$, there is at least one path that traverses from input to output layer.  (Section \ref{sec:the sub-graph})
    \end{theorem}
    \begin{proof}
        This conclusion can be proved by induction. 
        Let's denote $\mathcal{G}_{\operatorname{sub}}^{N}$ as a sub-graph with $N$ layers. If $N=2$ and $\alpha_{i | I =1}$ is remained, it is apparent that the $i$-th component in layer 1 leads to all components in layer 2. If $N=M+1$ and $\alpha_{m|M}$ is remained, the $m$-th component in layer $M$ leads to all components in layer $N$. As we have a path that begins from input layer and ends at all components in layer $M$, it can be certainly extended to layer $N$. 
    \end{proof}
    
    \subsection{Proof for the Sparsity of Sub-Graph}
    
    \begin{theorem}
        $\mathcal{L}_{\alpha}^{(2)}$ is minimal when there is only one instance of $\alpha_{\cdot \mid I}$ that is remained on each layer $I$.  (Section \ref{sec:two parallel objectives})
    \end{theorem}
    \begin{proof}
        As we restrict all $\alpha$ to be non-negative and their sum equals to $1$ on each layer, each term of  $\mathcal{L}_{\alpha}^{(2)}$ is greater or equal to zero, that is: 
        \begin{equation*}
            -\alpha_{i \mid I} \cdot \log \left( \alpha_{i \mid I}\right) \geq 0, ~ \forall i, \forall I
        \end{equation*}
        Therefore, $\mathcal{L}_{\alpha}^{(2)}$ is minimal when all of its terms are zero, which means each $\alpha$ has to be either zero or one. Under the constraint of $\alpha$, this condition is satisfied only when there is one instance of $\alpha$ that is remained on each layer. 
    \end{proof}
    
    \subsection{The Hyperparameters}
    
    Table \ref{table:params} lists the set of hyperparameters in our experiments. Most of these hyperparameters are common in previous DRL-based ATSC works \cite{survey_rl_on_atsc}. For TinyLight, the values for $D_{L_2}$  and $D_{L_3}$ are empirically determined with the principle that they are neither too small (to ensure the representation power of the policy network) nor too large (to decrease the resource consumption). 
    
    \begin{table*}[h]
        \centering
        \begin{tabular}{
                p{.1\linewidth}
                >{\raggedleft\arraybackslash}p{.2\linewidth} 
                p{.62\linewidth}
            }
            \toprule
            \textbf{Symbol} & \textbf{Value} & \textbf{Description} \\ 
            \midrule
            \multicolumn{3}{l}{\circled{1} Parameters for all models.} \\ 
            \quad -- $B$  & $ 100,000 $  & The maximal size of replay buffer.  \\ 
            \quad -- $M$ & $ 32 $ & The minibatch size for training.  \\ 
            \quad -- $\gamma $ & $ 0.9 $ & The discount factor for future rewards.  \\  
            \quad -- $\epsilon $ & 0.0\textasciitilde0.1 & The exploration rate, decreasing with respect to the number of episodes.  \\ 
            \quad -- $\tau$ & $0.1$ & The update ratio for the target Q-network.   \\ 
            \quad -- lr  & 1e-3 & The learning rate.  \\ 
            \multicolumn{3}{l}{\circled{2} Parameters for TinyLight.} \\ 
            \quad -- $\beta$ & $16.0$ & The weight of penalty term for the entropy of $\alpha$. \\ 
            \quad -- $D_{L_2}$ & $[16, ~18, ~20, ~22, ~24]$ & The list of output dimensions for layer 2 components in the super-graph. \\ 
            \quad -- $D_{L_3}$ & $[16, ~18, ~20, ~22, ~24]$ & The list of output dimensions for layer 3 components in the super-graph. \\                         
            \bottomrule
        \end{tabular}
        \caption{The list of hyperparameters in our experiments.  }
        \label{table:params}
    \end{table*}
    
    \subsection{The Training Environment}
    
    The experiments are conducted on an Ubuntu 16.04 workstation with 32 CPU cores (Intel Xeon E5-2620), 160G memory and 4 GPU cards (Titan-Xp). 
    All our software dependencies are listed in the \texttt{README} file of our source code, attached in another supplementary material. The evaluation on hardware is performed on an Arduino UNO Rev3\footnote{\url{https://store.arduino.cc/products/arduino-uno-rev3/}} board, with ATmega328P\footnote{\url{https://www.microchip.com/en-us/product/ATmega328P}} as the default MCU (Figure \ref{fig:mcu}). All the experimental results are performed on 10 parallel runs to obtain the distributional information. The training time for each intersection takes about 0.75 hour. The inference time of TinyLight (with post-training quantization) on ATmega328P is about 18.78ms per decision. 
    
    \begin{figure}[h]
        \centering
        \includegraphics[width=0.5\linewidth]{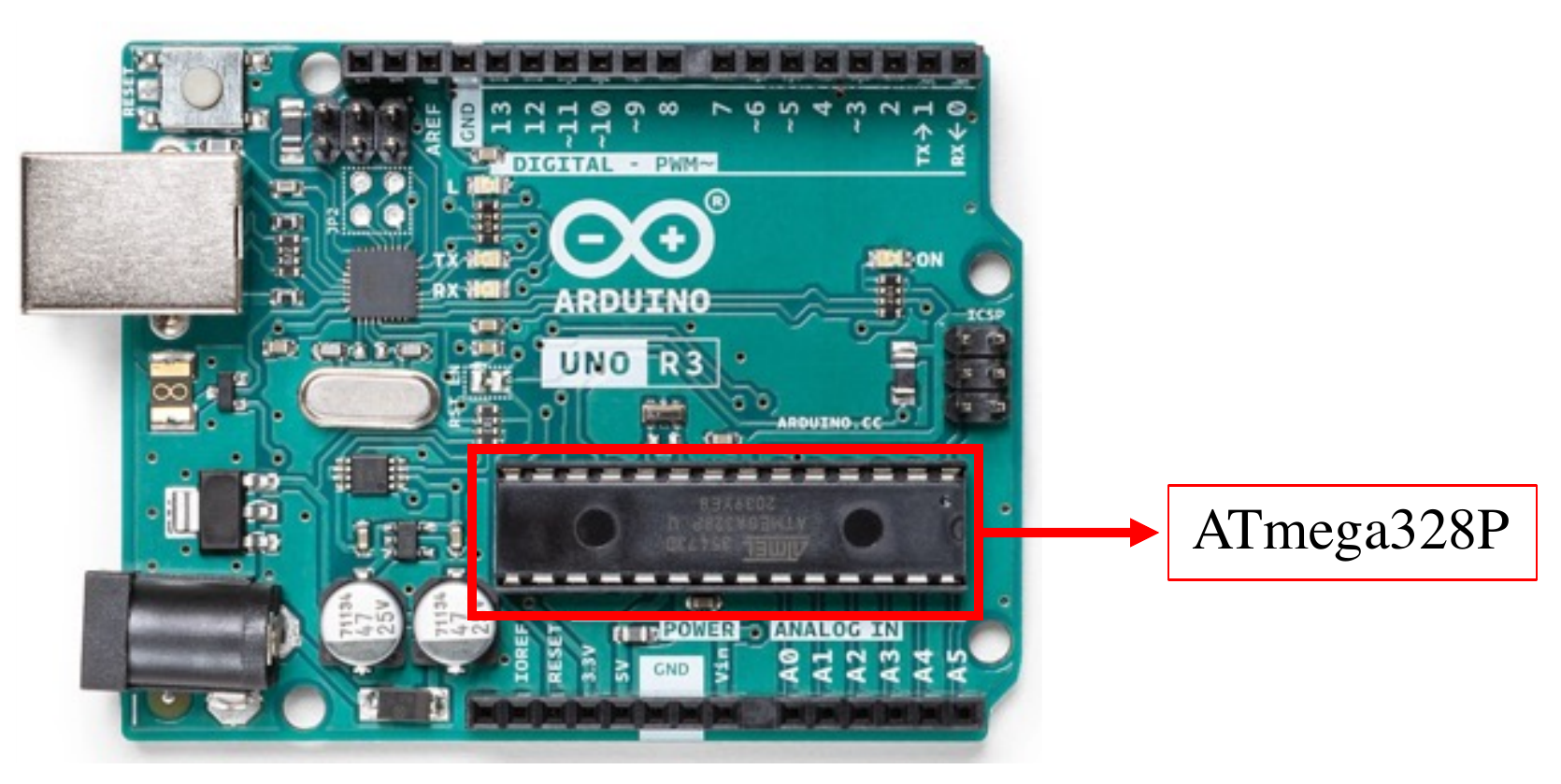}
        \caption{The Arduino UNO Rev3 and the ATmega328P (source image from Arduino's official website).  }
        \label{fig:mcu}
    \end{figure}
    
    \newpage 
    \section{Resource Consumption of Each Model}\label{appendix:resource}
    
    This section provides a detailed list of parameter size and FLOPs (floating point operations) for each model. Section \ref{sec: basic structure} enumerates the relevant atomic structures and operations. From Section \ref{sec:colight} to \ref{sec: tinylight}, we elaborate on the detailed implementation of CoLight \cite{colight_wei}, EcoLight \cite{ecoLight_Chauhan}, FixedTime \cite{fixed_cycle}, FRAP \cite{frap_zheng}, MaxPressure \cite{max_pressure}, MPLight \cite{mplight_aaai}, SOTL \cite{sotl} and TinyLight (this work), respectively. To make our comparison more intuitive, we choose one intersection from Jinan dataset to instantiate the parameter size and FLOPs of each model. Table \ref{table: meta info hangzhou intersection} shows the meta information of this intersection. 
    
    \begin{table*}[h]
        \centering
        \begin{tabular}{
                p{.08\linewidth}
                >{\raggedleft\arraybackslash}p{.08\linewidth} 
                p{.75\linewidth}
            }
            \toprule
            \textbf{Symbol} & \textbf{Value} & \textbf{Description} \\ 
            \midrule
            $L_{\operatorname{in}} $ & $ 12 $ & The number of incoming lanes.  \\  
            $L_{\operatorname{out}} $ & $ 12 $ & The number of outgoing lanes.  \\ 
            $P$  & $ 9 $  & The number of feasible phases.  \\ 
            $L_{\operatorname{link}}$ & $ 36 $ & The number of lane-links. Each lane-link includes an incoming lane and an outgoing lane.  \\ 
            $N$ & $5$ & The total number of relevant intersections, including the target intersection and all its adjacent intersections.  \\ 
            \bottomrule
        \end{tabular}
        \caption{The meta information of our chosen intersection from Jinan dataset. }
        \label{table: meta info hangzhou intersection}
    \end{table*}

    \subsection{Atomic structures and operations}\label{sec: basic structure}
    
    Table \ref{table: atomic structure and ops} lists the parameter size and FLOPs of relevant atomic structures and operations. The signature of each structure (operation) follows the PyTorch API, with some modifications to make the representation of parameter size and FLOPs more clear. The structures (Linear and Conv2d) have names that begin with a capital letter. The operations (relu, matmul and softmax) have names that are fully lowercase, and they have no parameters. For each structure (operation), its FLOPs are counted by the number of floating point arithmetic, such as addition, subtraction, multiplication, division, comparison and exponentiation. We focus on the on-device performance of each model at runtime. Therefore, it is assumed that the input of each model is a single instance, \textit{i.e.}, a mini-batch with a capacity of 1. 
    
    \begin{table*}[!h]
        \centering 
        \begin{tabular}{
                p{.38\linewidth}    
                p{.22\linewidth}
                p{.3\linewidth}
            }
            \toprule 
            \textbf{Structure/Operation} & \textbf{Parameter Size} & \textbf{FLOPs} \\ 
            \midrule 
            {\circled{1} Linear(in\_features=$F_{\operatorname{in}}$, } 
            & {$(F_{\operatorname{in}} + 1) \times F_{\operatorname{out}} $} 
            & {$(2 \times F_{\operatorname{in}} + 1) \times F_{\operatorname{out}}$} \\ 
            
            {\phantom{\circled{1} Linear(}out\_features=$F_{\operatorname{out}}$)} & & \\ 
            
            {\circled{2} Conv2d(in\_channels=$C_{\operatorname{in}}$, }  \rule{0mm}{5mm}
            &  {$(C_{\operatorname{in}} + 1) \times C_{\operatorname{out}}$}
            &  {$(2 \times C_{\operatorname{in}} + 1) \times C_{\operatorname{out}} \times H_{\operatorname{in}} \times W_{\operatorname{in}}$} \\ 
            
            {\phantom{\circled{2} Conv2d(}out\_channels=$C_{\operatorname{out}}$, } 
            &
            & \\ 
            
            {\phantom{\circled{2} Conv2d(}in\_width=$W_{\operatorname{in}}$,} 
            &
            & \\ 
            
            {\phantom{\circled{2} Conv2d(}in\_height=$H_{\operatorname{in}}$,} 
            &
            & \\             
            
            {\phantom{\circled{2} Conv2d(}kernel\_size=$1 \times 1$) } 
            &
            & \\

            {\circled{3} relu(features=$F$)} \rule{0mm}{5mm} 
            & $-$ 
            & $2 \times F$ \\ 
            
            {\circled{4} matmul(input=$\operatorname{Mat}_{X\times Y}$, }  \rule{0mm}{5mm}                      
            & {$-$} 
            & {$2 \times Y \times X \times Z$} \\

            {\phantom{\circled{4} matmul(}other=$\operatorname{Mat}_{Y\times Z}$) }                     
            &
            & \\             
            
            {\circled{5} softmax(features=$F$)}  \rule{0mm}{5mm}                      
            & $-$ 
            & $3 \times F$ \\ 
            
            \bottomrule
        \end{tabular}
        \caption{The parameter size and FLOPs of relevant atomic structures and operations. }
        \label{table: atomic structure and ops}
    \end{table*}
    
    The following are supplementary notes to Table \ref{table: atomic structure and ops}: 
    \begin{itemize}
        \item \textbf{Linear}: The weight matrix occupies a size of $F_{\operatorname{in}} \times F_{\operatorname{out}}$ and the bias occupies a size of $F_{\operatorname{out}}$. For each output feature, the computation involves 1 multiplication and 1 addition on each dimension of input feature (with the weight matrix) and 1 addition for the bias. 
        \item \textbf{Conv2d}: For the implementation of each comparing model, only Conv2d with kernel size $1 \times 1$ is involved. For all elements in the output, the convolution kernel is shared, but the computation is performed separately. This makes the parameter size of Conv2d similar to that of Linear with $C_{\operatorname{in}}$ input feature and $C_{\operatorname{out}}$ output feature. However, its FLOPs are scaled by a factor of $H_{\operatorname{in}} \times W_{\operatorname{in}}$ because of its separated computation. 
        \item \textbf{relu}: For each feature, the computation involves 1 comparison and 1 assignment. 
        \item \textbf{matmul}: For each element of the output matrix, the computation involves $Y$ multiplications and $Y$ additions. The output matrix has $X \times Z$ elements in total. 
        \item \textbf{softmax}: The computation involves $F$ exponentiations, $F$ additions and $F$ divisions. 
    \end{itemize}
    
    \subsection{CoLight}\label{sec:colight}
    
    CoLight \cite{colight_wei} incorporates the influences of neighboring intersections to facilitate the decision making of target intersection. The authors implement this idea with an index-free graph attention network, whose parameters are shared among all the intersections to reduce the model size. Table \ref{table: colight meta information} lists the meta information of CoLight, which is obtained from the original paper and its released code\footnote{\url{https://github.com/wingsweihua/colight}}. The number of attention head is set to 5 according to the experiments in Section 5.7.3 of \cite{colight_wei}. Under this setup, the parameter size and FLOPs of CoLight are presented in Table \ref{table: colight model size} and \ref{table: colight flops}, respectively. 
    
    \begin{table*}[!h]
        \centering 
        \begin{tabular}{lrl}
            \toprule
            \textbf{Symbol} & \textbf{Value} & \textbf{Description}  \\ 
            \midrule
            $D_{\operatorname{emb.}}$ & $32$ & The dimension of observation embedding. \\ 
            $D_{\operatorname{att.}}$   & $32$ & The dimension of each attention head. \\ 
            $D_{\operatorname{rep.}}$  & $32$ & The dimension of hidden representation. \\ 
            $H$                                   & $5$  & The number of attention head. \\ 
            \bottomrule
        \end{tabular}
        \caption{The meta information of CoLight. }
        \label{table: colight meta information}
    \end{table*}    
    
    \begin{table*}[!h]
        \centering 
        \begin{tabular}{
                @{}p{.29\linewidth}
                p{.4\linewidth} 
                >{\raggedleft\arraybackslash}p{.25\linewidth}@{}
            }
            \toprule
            \textbf{Description} & \textbf{Parameter Size} & \textbf{Instantiation} \\ 
            \midrule 
            
            \multicolumn{3}{l}{\circled{1} The observation embedding layer.}  \\ 
            \quad -- Layer 1 to layer 2 & 
            Linear([$L_{\operatorname{in}}$, $L_{\operatorname{out}}$, $P$], $D_{\operatorname{emb.}}$)
            & $34 \times 32 = 1,088$ \\ 
            \quad -- Layer 2 to layer 3 & 
            Linear($D_{\operatorname{emb.}}$, $D_{\operatorname{emb.}}$) 
            & $33 \times 32 = 1,056$ \\ 
            
            \multicolumn{3}{l}{\circled{2} The mapping of observation embedding to $H$ attention heads. \rule{0pt}{5mm}}  \\ 
            \quad -- $H$ heads & 
            Linear($D_{\operatorname{emb.}}$, $D_{\operatorname{att.}}$) $\times H$
            & $33 \times 32 \times 5 = 5,280$ \\  
            
            \multicolumn{3}{l}{\circled{3} The mapping of $N$ neighbor embeddings to $H$ attention heads. \rule{0pt}{5mm}}  \\ 
            \quad -- $H$ heads            & 
            Linear($D_{\operatorname{emb.}}$, $D_{\operatorname{att.}}$) $\times H$ 
            & $33 \times 32 \times 5 = 5,280$ \\  
            
            \multicolumn{3}{l}{\circled{4} The mapping of embedding to hidden representation w.r.t. each attention head. \rule{0pt}{5mm}}  \\ 
            \quad -- $H$ heads            & 
            Linear($D_{\operatorname{emb.}}$, $D_{\operatorname{rep.}}$) $\times H$
            & $33 \times 32 \times 5 = 5,280$ \\  
            
            \multicolumn{3}{l}{\circled{5} The Q-value prediction layer. \rule{0pt}{5mm}}  \\ 
            \quad -- Layer 1 to layer 2 & Linear($D_{\operatorname{rep.}}$, $D_{\operatorname{rep.}}$) & $33 \times 32 = 1,056$  \\ 
            \quad -- Layer 2 to layer 3 & Linear($D_{\operatorname{rep.}}$, $P$) &  $33 \times 9 = 297$\\
            \midrule 
            \textbf{Total} & & \textbf{19,337} \\ 
            \bottomrule
        \end{tabular}
        \caption{The parameter size of CoLight.}
        \label{table: colight model size}
    \end{table*}

    \begin{table*}[!h]
        \centering
        \begin{tabular}{
                @{}p{.29\linewidth}
                p{.4\linewidth} 
                >{\raggedleft\arraybackslash}p{.25\linewidth}@{}
            }
            \toprule 
            \textbf{Description} & \textbf{FLOPs} & \textbf{Instantiation}\\ 
            \midrule 
            \multicolumn{3}{l}{\circled{1} The observation embedding layer.}  \\
            
            {\quad - Layers for the target}
            & relu({Linear}([$L_{\operatorname{in+out}}$, $P$], $D_{\operatorname{emb.}}$))
            & $67 \times 32 + 64 = 2,208$\\              
            & relu({Linear}($D_{\operatorname{emb.}}$, $D_{\operatorname{emb.}}$))
            & $65 \times 32 + 64 = 2,144$\\
            
            {\quad - Layers for its neighbors}
            & relu({Linear}([$L_{\operatorname{in+out}}$, $P$], $D_{\operatorname{emb.}}$)) \rdelim\}{2}{0mm}[{$\times (N -1)$}]
            & $(2,208 + 2,144) \times 4$
            \\ 
            \quad 
            &{relu}({Linear}($D_{\operatorname{emb.}}$, $D_{\operatorname{emb.}}$)) 
            & $= 17,408$\\ 
            
            \multicolumn{3}{l}{\circled{2} The mapping of observation embedding to $H$ attention heads. \rule{0pt}{5mm}}  \\
            \quad - $H$ heads
            & {relu}({Linear}($D_{\operatorname{emb.}}$, $D_{\operatorname{att.}}$)) $\times H$ 
            & $2,144 \times 5 = 10,720$\\ 
            
            \multicolumn{3}{l}{\circled{3} The mapping of $N$ neighbor embeddings to $H$ attention heads. } \rule{0pt}{5mm} \\ 
            \quad - Dense layer 
            & {relu}({Linear}($D_{\operatorname{emb.}}$, $D_{\operatorname{att.}}$)) 
            $\times N$ 
            ~~~\rdelim\}{3}{6mm}[$\times H$] 
            & $10,720 \times 5 = 53,600$ \\ 
            \quad - Computing logits
            & {matmul}(Mat$_{N \times D_{\operatorname{att.}}}$, Mat$_{D_{\operatorname{att.}} \times 1}$)
            & $2 \times 32 \times 5 \times 5 = 1,600$ \\ 
            
            \quad - Softmax 
            & {softmax}(N) 
            & $3 \times 5 \times 5 = 75$ \\ 
            
            \multicolumn{3}{l}{\circled{4} The mapping of embedding to hidden representation w.r.t. each attention head. \rule{0pt}{5mm}}  \\ 
            \quad - Dense layer 
            & {relu}({Linear} ($D_{\operatorname{emb.}}$, $D_{\operatorname{rep.}}$)) ~~~~~~\rdelim\}{2}{6mm}[$\times H$]
            & $2,144 \times 5 = 10,720$ \\ 
            
            \quad - Weighting by attention 
            & matmul(Mat$_{1 \times N}$, Mat$_{N \times D_{\operatorname{rep.}}}$)
            & $2 \times 32 \times 5 \times 5 = 1,600$ \\ 
            
            \quad - Averaging over heads 
            & $(H_{\operatorname{sum}} + 1_{\operatorname{division}}) \times D_{\operatorname{rep.}}$ 
            & $(5 + 1) \times 32 = 192$\\ 
            
            \multicolumn{3}{l}{\circled{5} The Q-value prediction layer. \rule{0pt}{5mm}}  \\ 
            \quad - Layer 1 to layer 2
            & relu(Linear($D_{\operatorname{rep.}}$, $D_{\operatorname{rep.}}$))
            & $65 \times 32 + 64 = 2,144$ \\ 
            \quad - Layer 2 to layer 3
            & Linear($D_{\operatorname{rep.}}$, $P$)
            & $65 \times 9 = 585$ \\

            \midrule
            \textbf{Total} & & \textbf{102,996} \\
            \bottomrule 
        \end{tabular}
        \caption{The FLOPs of CoLight. }
        \label{table: colight flops}
    \end{table*}

    \subsection{EcoLight} 
    
    EcoLight \cite{ecoLight_Chauhan} is designed for scenarios with extreme budget and network constraints, which partially shares the goal of TinyLight. The authors first construct a model that has a moderate size of parameters. Then, for the deployment in reality, they replace the model with a series of artificially designed thresholds. This simplifies the deployment. However, it also limits the model's representation power.  In this work, we report the result of the original version of EcoLight (defined in Section 3 of \cite{ecoLight_Chauhan}), which stands for its upperbound performance. This version takes two features as input, which are the traffic density on lanes with green signal and red signal, respectively. Its network is constructed with four fully-connected layers. The first layer is the input feature, and the remaining three layers have a shape of 10, 10 and 2, respectively. We present the parameter size and FLOPs of EcoLight  in Table \ref{table: ecolight model size} and \ref{table: ecolight flops}. 
    
    \begin{table*}[!h]
        \centering 
        \begin{tabular}{
                @{}p{.29\linewidth}
                p{.4\linewidth} 
                >{\raggedleft\arraybackslash}p{.25\linewidth}@{}
            }
            \toprule
            \textbf{Description} & \textbf{Parameter Size} & \textbf{Instantiation} \\ 
            \midrule 
            
            Layer 1 to layer 2 & 
            Linear($2$, $10$)
            & $(2 + 1) \times 10 = 30$ \\ 
            Layer 2 to layer 3& 
            Linear($10$, $10$) 
            & $(10 + 1) \times 10 = 110$ \\ 
            
            Layer 3 to output layer & 
            Linear($10$, $2$) 
            & $(10 + 1) \times 2 = 22$ \\
            \midrule 
            \textbf{Total} & & \textbf{162} \\ 
            \bottomrule
        \end{tabular}
        \caption{The parameter size of EcoLight.}
        \label{table: ecolight model size}
    \end{table*}    
    
    \begin{table*}[!h]
        \centering 
        \begin{tabular}{
                @{}p{.29\linewidth}
                p{.4\linewidth} 
                >{\raggedleft\arraybackslash}p{.25\linewidth}@{}
            }
            \toprule
            \textbf{Description} & \textbf{FLOPs} & \textbf{Instantiation} \\ 
            \midrule 
            
            Layer 1 to layer 2 & 
            relu(Linear($2$, $10$))
            & $(2 \times 2 + 3) \times 10 = 70$ \\ 
            Layer 2 to layer 3& 
            relu(Linear($10$, $10$)) 
            & $(2 \times 10 + 3) \times 10 = 230$ \\ 
            
            Layer 3 to output layer & 
            Linear($10$, $2$) 
            & $(2 \times 10 + 1) \times 2 = 42$ \\
            \midrule 
            \textbf{Total} & & \textbf{342} \\ 
            \bottomrule
        \end{tabular}
        \caption{The FLOPs of EcoLight.}
        \label{table: ecolight flops}
    \end{table*}    
    
    \subsection{FixedTime}
    
    FixedTime \cite{fixed_cycle} switches phases with regular and consistent intervals, which is widely adopted in the real world. The idea of FixedTime can be realized with a simple digital timer. Therefore, its parameter size and FLOPs are ignored in this work. 
    
    \subsection{FRAP} 
    
    FRAP \cite{frap_zheng} aims to improve the transportation efficiency with a procedure called \textit{phase competition}. Briefly, this procedure compares phases in pairs and gives priority to the one with larger traffic movement. In this way, the model becomes invariant to symmetric operations such as flipping and rotation on the target intersection. To implement this idea, the authors design a cube to store the embedding of all phase pairs. Based on this cube, the relationships between each pair of competing phases are extracted with a series of $1\times 1$ convolution kernels. We first list the meta information of FRAP in Table \ref{table: frap meta information}. Then, its parameter size and FLOPs are presented in Table \ref{table: frap model size} and \ref{table: frap flops}, respectively. This implementation is in alignment with the authors' released code\footnote{\url{https://github.com/gjzheng93/frap-pub}}. 
    
    \begin{table*}[!h]
        \centering 
        \begin{tabular}{lrl}
            \toprule
            \textbf{Symbol} & \textbf{Value} & \textbf{Description}  \\ 
            \midrule
            $D_{\operatorname{emb.}}$ & $4$ & The dimension of observation embedding. \\ 
            $D_{\operatorname{rep.}}$  & $16$ & The dimension of hidden representation. \\ 
            $D_{\operatorname{2\times rep.}}$  & $32$ & The total dimension of a pair of competing phases. \\ 
            $D_{\operatorname{conv.}}$  & $20$ & The dimension of intermediate representation for the output of Conv2d. \\ 
            \bottomrule
        \end{tabular}
        \caption{The meta information of FRAP. }
        \label{table: frap meta information}
    \end{table*}

    \begin{table*}[!h]
        \centering 
        \begin{tabular}{
                @{}p{.34\linewidth}
                p{.4\linewidth} 
                >{\raggedleft\arraybackslash}p{.2\linewidth}@{}
            }
            \toprule
            \textbf{Description} & \textbf{Parameter Size} & \textbf{Instantiation} \\ 
            \midrule 
            
            \multicolumn{3}{l}{\circled{1} The observation embedding layer.}  \\ 
            \quad -- Phase embedding & 
            Linear(1, $D_{\operatorname{emb.}}$)
            & $2 \times 4 = 8$ \\ 
            \quad -- Vehicle embedding & 
            Linear(1, $D_{\operatorname{emb.}}$) 
            & $2 \times 4 = 8$ \\ 
            \quad -- Lane-link embedding & 
            Linear([$D_{\operatorname{emb.}}$, $D_{\operatorname{emb.}}$], $D_{\operatorname{rep.}}$)
            & $9 \times 16 = 144$ \\ 
            \quad -- Relationship embedding & 
            Linear(1, $D_{\operatorname{emb.}}$)
            & $2 \times 4 = 8$ \\ 
            
            \multicolumn{3}{l}{\circled{2} The convolution kernels.  \rule{0pt}{5mm}}  \\ 
            \quad - Pair of phases & 
            Conv2d($D_{\operatorname{2\times rep.}}$, $D_{\operatorname{conv.}}$)
            & $33 \times 20 = 660$ \\  
            \quad - Competition mask            & 
            Conv2d($D_{\operatorname{emb.}}$, $D_{\operatorname{conv.}}$) 
            & $5 \times 20 = 100$ \\  
            
            \multicolumn{3}{l}{\circled{3} The Q-value prediction layer. \rule{0pt}{5mm}}  \\ 
            \quad -- Layer 1 to layer 2 & Conv2d($D_{\operatorname{conv.}}$, $D_{\operatorname{conv.}}$) & $21 \times 20 = 420$  \\ 
            \quad -- Layer 2 to layer 3 & Conv2d($D_{\operatorname{conv.}}$, $1$) &  $21 \times 1 = 21$\\
            \midrule 
            \textbf{Total} & & \textbf{1,369} \\ 
            \bottomrule
        \end{tabular}
        \caption{The parameter size of FRAP.}
        \label{table: frap model size}
    \end{table*}

    \begin{table*}[!h]
        \centering
        \begin{tabular}{@{}
                p{.31\linewidth}
                p{.4\linewidth} 
                >{\raggedleft\arraybackslash}p{.23\linewidth}@{}
            }
            \toprule 
            \textbf{Description} & \textbf{FLOPs} & \textbf{Instantiation}\\ 
            \midrule 
            
            \multicolumn{3}{l}{\circled{1} The observation embedding layer.}  \\ 
            \quad -- Phase embedding & 
            relu(Linear(1, $D_{\operatorname{emb.}}$))
            & $5 \times 4 = 20$ \\ 
            \quad -- Vehicle embedding & 
            relu(Linear(1, $D_{\operatorname{emb.}}$)) 
            & $5 \times 4 = 20$ \\ 
            \quad -- Lane-link embedding & 
            relu(Linear([$D_{\operatorname{emb.}}$, $D_{\operatorname{emb.}}$], $D_{\operatorname{rep.}}$))
            & $19 \times 16 = 304$ \\ 
            \quad -- Relationship embedding & 
            relu(Linear(1, $D_{\operatorname{emb.}}$))
            & $5 \times 4 = 20$ \\ 
            
            \multicolumn{3}{l}{\circled{2} The convolution operation.  \rule{0pt}{5mm}}  \\ 
            \quad - Pair of phases & 
            Conv2d($D_{\operatorname{2\times rep.}}$, $D_{\operatorname{conv.}}$, $P$, $P - 1$)
            & $65 \times 20 \times 9 \times  8 = 93,600$ \\  
            \quad - Competition mask            & 
            Conv2d($D_{\operatorname{emb.}}$, $D_{\operatorname{conv.}}$, $P$, $P - 1$) 
            & $9 \times 20 \times 9 \times 8 = 12,960$ \\  
            
            \multicolumn{3}{l}{\circled{3} The Q-value prediction layer. \rule{0pt}{5mm}}  \\ 
            \quad -- Layer 1 to layer 2 & relu(Conv2d($D_{\operatorname{conv.}}$, $D_{\operatorname{conv.}}$, $P$, $P - 1$)) & $41 \times 20 \times 9 \times 8 + 2 \times 20 \times 9 \times 8 = 61,920$  \\ 
            \quad -- Layer 2 to layer 3 & Conv2d($D_{\operatorname{conv.}}$, $1$, $P$, $P - 1$) &  $41 \times 1 \times 9 \times 8 = 2,952$\\            
            
            \multicolumn{3}{l}{\circled{4} Intermediate operations. \rule{0pt}{5mm}}  \\
            \quad -- Phase-based aggregation  
            & matmul(Mat$_{D_{\operatorname{rep.}} \times L_{\operatorname{link}}}$, Mat$_{L_{\operatorname{link}} \times P}$) 
            & $2 \times 36 \times 16 \times 9 = 10,368$\\ 
            
            \quad -- Mask over cube 
            & $D_{\operatorname{rep.}} \times P \times (P - 1)$
            & $20 \times 9 \times 8 = 1,440$\\ 
            
            \quad -- Summation over phase 
            &  $P \times (P - 1)$
            &  $9 \times 8 = 72$ \\ 
            
            \midrule
            \textbf{Total} & & \textbf{183,676} \\
            \bottomrule 
        \end{tabular}
        \caption{The FLOPs of FRAP.}
        \label{table: frap flops}
    \end{table*}

    \subsection{MaxPressure}
    
    MaxPressure \cite{max_pressure} is a rule-based model, which chooses next phase by the one with the largest pressure. This can be implemented with $P$ accumulators. Its computation involves $L_{\operatorname{in}} + L_{\operatorname{out}}$ summations and $P$ comparisons for each decision. 
    
    \subsection{MPLight}
    
    MPLight \cite{mplight_aaai} addresses to solve the adaptive traffic signal control problem that is of large scale. This model adopts FRAP \cite{frap_zheng} as its prototype and modify it by (1) using pressure to coordinate with neighboring intersections, and (2) using parameter sharing to accelerate its training. For the implementation of MPLight on a single intersection, these modifications do not introduce additional overhead in computation and storage when comparing to FRAP. Therefore, MPLight shares the same parameter size and FLOPs with FRAP. 
    
    \subsection{SOTL} 
    
    SOTL \cite{sotl} is a rule-based model, which counts the vehicles on incoming lanes that have (or not have) the right of way with respect to current phase. These two numbers are then compared with a pair of predefined thresholds to determine whether it is appropriate to switch to the next phase. Therefore,  the parameters of SOTL consist of 2 predefined thresholds. Its computation takes $L_{\operatorname{in}}$ additions and 2 comparisons in total. 
    
    \subsection{TinyLight}\label{sec: tinylight}
    
    The output of TinyLight is a four-layer perceptron with two features. To minimize the resource consumption, TinyLight only involves basic structures/operations such as Linear and relu. We sample an instance of TinyLight from our experiments in Jinan dataset. Table \ref{table: tinylight meta information} lists the meta information of our model. Its parameter size and FLOPs are presented in Table \ref{table: tinylight model size} and \ref{table: tinylight flops}, respectively. 
    
    \begin{table*}[!h]
        \centering 
        \begin{tabular}{lrl}
            \toprule
            \textbf{Symbol} & \textbf{Value} & \textbf{Description}  \\ 
            \midrule
            $D_{\operatorname{1-fea.}}$ & $12$ & The dimension of first feature. \\ 
            $D_{\operatorname{2-fea.}}$  & $9$ & The dimension of second feature. \\ 
            $D_{\operatorname{2-lay.}}$  & $18$ & The dimension of second layer. \\ 
            $D_{\operatorname{3-lay.}}$  & $20$ & The dimension of third layer. \\ 
            \bottomrule
        \end{tabular}
        \caption{The meta information of TinyLight. }
        \label{table: tinylight meta information}
    \end{table*}    
    
    \begin{table*}[!h]
        \centering 
        \begin{tabular}{
                @{}p{.32\linewidth}
                p{.4\linewidth} 
                >{\raggedleft\arraybackslash}p{.22\linewidth}@{}
            }
            \toprule
            \textbf{Description} & \textbf{Parameter Size} & \textbf{Instantiation} \\ 
            \midrule 
            
            Feature 1 to layer 2 & 
            Linear($D_{\operatorname{1-fea.}}$, $D_{\operatorname{2-lay.}}$)
            & $(12 + 1) \times 18 = 234$ \\ 
            Feature 2 to layer 2 & 
            Linear($D_{\operatorname{2-fea.}}$, $D_{\operatorname{2-lay.}}$)
            & $(10 + 1) \times 18 = 198$ \\
            Layer 2 to layer 3 & 
            Linear($D_{\operatorname{2-lay.}}$, $D_{\operatorname{3-lay.}}$)
            & $(18 + 1 ) \times 20 = 380 $\\  
            Layer 3 to output layer & 
            Linear($D_{\operatorname{3-lay.}}$, $P$) 
            & $(20 + 1) \times 9 = 189$ \\
            \midrule 
            \textbf{Total} & & \textbf{1,001} \\ 
            \bottomrule
        \end{tabular}
        \caption{The parameter size of TinyLight.}
        \label{table: tinylight model size}
    \end{table*}    
    
    \begin{table*}[!h]
        \centering 
        \begin{tabular}{
                @{}p{.32\linewidth}
                p{.4\linewidth} 
                >{\raggedleft\arraybackslash}p{.22\linewidth}@{}
            }
            \toprule
            \textbf{Description} & \textbf{FLOPs} & \textbf{Instantiation} \\ 
            \midrule 
            
            Feature 1 to layer 2 & 
            relu(Linear($D_{\operatorname{1-fea.}}$, $D_{\operatorname{2-lay.}}$))
            & $(2 \times 12 + 3) \times 18 = 486$ \\ 
            Feature 2 to layer 2 & 
            relu(Linear($D_{\operatorname{2-fea.}}$, $D_{\operatorname{2-lay.}}$))
            & $(2 \times 9 + 3) \times 18 = 378$ \\ 
            Element-wise sum on layer 2 & 
            $D_{\operatorname{2-lay.}}$ 
            & 18 \\ 
            Layer 2 to layer 3& 
            relu(Linear($D_{\operatorname{2-lay.}}$, $D_{\operatorname{3-lay.}}$)) 
            & $(2 \times 18 + 3) \times 20 = 780$ \\ 
            
            Layer 3 to output layer & 
            Linear($D_{\operatorname{3-lay.}}$, $P$) 
            & $(2 \times 20 + 1) \times 9 = 369$ \\
            \midrule 
            \textbf{Total} & & \textbf{2,031} \\ 
            \bottomrule
        \end{tabular}
        \caption{The FLOPs of TinyLight.}
        \label{table: tinylight flops}
    \end{table*}    
\end{document}